\documentclass[11pt]{article}
\usepackage[a4paper, total={6in, 8in}]{geometry}
\usepackage{amsmath}
\usepackage{amssymb}
\usepackage{latexsym}
\usepackage{caption}
\usepackage{amsmath, amsfonts,amssymb, amsthm, euscript,makeidx,color,mathrsfs}
\usepackage{verbatim}
\usepackage{hyperref}
\usepackage{graphicx}
 \usepackage[inline]{trackchanges}
\usepackage{algorithm}
\usepackage[noend]{algpseudocode}

\newtheorem{remark}{Remark}

\newtheorem{assumption}{Assumption}
\def\qed{ \ \vrule width.2cm height.2cm depth0cm\smallskip}

\newcommand{\hP}{\hat\dbP}
\newcommand{\sig}{\mathsf{Sig}}

\newcommand{\oo}{\overline}
\newcommand{\uu}{\underline}

\newcommand{\ba}{\begin{array}}
\newcommand{\ea}{\end{array}}
\newcommand{\be}{\begin{equation}}
\newcommand{\ee}{\end{equation}}
\newcommand{\bea}{\begin{eqnarray}}
\newcommand{\eea}{\end{eqnarray}}
\newcommand{\beaa}{\begin{eqnarray*}}
\newcommand{\eeaa}{\end{eqnarray*}}

\def\dbR{\mathbb{R}}

\def\b{\beta}

\def\l{\lambda}

\def\th{\theta}
\def\o{\omega}

\def\D{\Delta}

\def\cF{{\cal F}}

\def\cK{{\cal K}}
\def\cL{{\cal L}}

\def\cQ{{\cal Q}}

\def\hE{\mathbb{E}}

\def\hG{\mathbb{G}}

\def\hN{\mathbb{N}}

\def\hP{\mathbb{P}}
\def\hQ{\mathbb{Q}}
\def\hR{\mathbb{R}}

\def\q{\quad}

\def\pa{\partial}

\def\qed{ \hfill \vrule width.25cm height.25cm depth0cm\smallskip}

\newcommand{\basa}{\begin{assumption}}
\newcommand{\easa}{\end{assumption}}

\newcommand{\bas}{\begin{assum}}
\newcommand{\eas}{\end{assum}}

\def\pa{\partial}

\def\1{{\bf 1}}

\def\:{\!:\!}

at 9pt

\newcommand{\ts}{\mathsf{T}}
\newcommand{\LS}{\mathsf{LS}}
\begin{document}

\newtheorem{thm}{Theorem}[section]
\newtheorem{lem}[thm]{Lemma}
\newtheorem{cor}[thm]{Corollary}
\newtheorem{prop}[thm]{Proposition}
\newtheorem{rem}[thm]{Remark}
\newtheorem{eg}[thm]{Example}
\newtheorem{defn}[thm]{Definition}
\newtheorem{assum}[thm]{Assumption}

\renewcommand {\theequation}{\arabic{section}.\arabic{equation}}
\def\thesection{\arabic{section}}
\title{\bf  Deep Signature FBSDE Algorithm}

\author{
Qi Feng\thanks{Department of Mathematics, University of Michigan, Ann Arbor, MI 48109-1043. Email: \url{qif@umich.edu} }
\and
Man Luo\thanks{Department of Mathematics, University of Southern California, Los Angeles, CA 90089-2532. Email: \url{manl@usc.edu}}
\and
Zhaoyu Zhang\thanks{Department of Mathematics, University of Southern  California, Los Angeles, CA 90089-2532. Email: \url{zzhang51@usc.edu}}
}

\date{}
\maketitle

\begin{abstract}
  We propose a deep signature/log-signature FBSDE algorithm to solve forward-backward stochastic differential equations (FBSDEs) with state and path dependent features. By incorporating the deep signature/log-signature transformation into the recurrent neural network (RNN) model, our algorithm shortens the training time, improves the accuracy, and extends the time horizon comparing to methods in the existing literature.  Moreover, our algorithms can be applied to a wide range of applications such as state and path dependent option pricing involving high-frequency data, model ambiguity, and stochastic games, which are linked to parabolic partial differential equations (PDEs), and path-dependent PDEs (PPDEs). Lastly, we also derive the convergence analysis of the deep signature/log-signature FBSDE algorithm.
\end{abstract}

\section{Introduction}
\paragraph{Motivation.} Recent developments of numerical algorithm for solving high dimensional PDEs draw a great amount of attention in various scientific fields. In the seminal paper \cite{weinan2017deep}, deep learning technique was first introduced to study the numerical algorithms for high dimensional parabolic PDEs. The deep learning BSDE method is based on the non-linear Feynman-Kac formula, which provides the equivalent relations between parabolic PDEs and Markovian backward stochastic differential equations (BSDEs) (see e.g. \cite{pardoux1992backward}). When the system does not have Markovian property, e.g. path-dependent property involved, the BSDE is equivalent to a path-dependent PDE (PPDE), which was first introduced in \cite{dupire2019functional} for path-dependent option pricing problem. 
The deep learning BSDE method has been recently extended to design numerical algorithms for PPDEs. The path-dependent property introduces extra complexity in the numerical scheme, and it returns a high dimensional problem even if the original space variable is low dimensional. In this study,  we shall focus on the numerical solutions for the corresponding Markovian and non-Markovian FBSDEs.

For the deep learning BSDE method \cite{weinan2017deep}, it shows the efficiency of machine learning in solving high dimensional parabolic PDEs but subject to small Lipschitz constants or equivalently small time duration. The exponential stopping time strategy has been introduced in \cite{ruan2020} to extend the time duration. However, both algorithms are still using the deep neural network combined with standard Euler scheme in essence, which makes it sensitive to the time discretization. Namely, the time dimension is still large for long time duration, which may take a long time to train the deep neural network (DNN) model. Furthermore, the deep learning BSDE method is not robust to missing data. If we miss a proportion of our data (e.g. data points in the Euler scheme), the accuracy will be affected. In particular, this is the same type of difficulty when dealing with high frequency data. In this case, one has to down-sample the stream data to a coarser time grid to feed it into the DNN-type algorithm. It may miss the microscopic characteristic of the streamed data and render lower accuracy. On the other hand, the high frequency and path-dependent features show up naturally in option pricing problems and non-linear expectations within various financial contexts, e.g. limit order book \cite{bielecki2013dynamic, carr2001pricing, cochrane2000beyond, hansen1991implications, Man2021, madan2010illiquid}, nonlinear pricing \cite{pardoux1992backward, yu2020backward}, Asian option pricing \cite{musiela_rutkowski_2005}, model ambiguity\cite{dirk2017good, biagini2017robust, cohen2017european, garlappi2007porfolio}, stochastic games and mean field games \cite{fouque_zhang, pham_zhang}, etc.

\paragraph{Our work.}
Motivated by these problems, we introduce the deep signature transformation into the recurrent neural network (RNN) model to solve BSDEs. The ``signature" is defined as an iterated integral of a continuous path with bounded $p$-variation, for $p>1$, which is a recurring theme in the rough path theory introduced by T. Lyons \cite{lyons1994differential}.  The ``signature" has recently been used to define kernels \cite{ CO18, KO16, min2020convolutional} for sequentially ordered data in the corresponding reproducing kernel Hilbert space (RKHS). This idea is further developed in \cite{kidger2019deep} to design ``deep signature" by combing the kernel method and DNN. Furthermore, the ``deep signature" has been used in RNN to study controlled differential equations in \cite{liao2019learning}. The signature approach also provides a non-parametric way for extraction of characteristic features from the data, see e.g. \cite{levin2013learning}. The data are converted into a multi-dimensional path through various embedding algorithms and then processed for computation of individual terms of the signature, which captures certain information contained in the data. The advantage of this signature method is that this method can deal with high frequency data, and is not sensitive to the time discretization. Motivated by this idea, we propose to combine the signature/log-signature transformation and RNN model to solve the FBSDEs,  which should have a much coarser time partition, a better downsampling effect, and more robust to the high-frequency data assumptions.
\paragraph{Related works.}
The numerical algorithm for solving PPDE with path dependent terminal condition (first type PPDE) has been recently studied in \cite{sabate2020solving, saporito2020pdgm} by using 
recurrent neural network. The second type PPDE arises from the Volterra SDE setting, where the non-Markovian property is introduced by the forward process instead of the terminal condition. The numerical algorithms for the option pricing problem in the Volterra SDEs setting has been recently studied in \cite{jacquier2019deep, ruan2020} by using deep learning, \cite{bayer2016pricing} by using regularity structures, and \cite{FengZhang2020} by using cubature formula. 

However, none of these works consider the high frequency data features in the algorithm. Neither do they consider the longer time duration in the model. Furthermore, we also provide the convergence analysis of our algorithm after introducing the signature/log-signature transformation layer into the RNN model. 

\section{Algorithms}
\subsection{Signature and signature transformation}
In this section, we introduce the preliminary facts about the signature from the rough path theory \cite{lyons1994differential} and the signature transformation \cite{liao2019learning} we used in the algorithm. In general, for a bounded variation path $x_t\in\hR^d$, for $t\in[0,T]$,  the  signature of $x$ (up to order $N$) is defined as the iterated integrals of $x$. More precisely, for a word $J=(j_1,\cdots,j_k)\in\{1,\cdots,d\}^k$ with size $|J|=k$,
\bea \label{signature intro}\sig_N(x)_t&=&\sum_{k=0}^N\int_{0<t_1<\cdots<t_k<t}dx_{t_1}\otimes\cdots\otimes dx_{t_k},\quad t\in[0,T],\\
&=&\Big(1,\sum_{j=1}^d \int_0^tdx^j_{t_1},\cdots,\sum_{|J|=N}\int_{0<t_1<\cdots<t_N<t}dx^{j_1}_{t_1}\cdots dx^{j_N}_{t_N} \Big)   \nonumber
\eea
where we use the convention that $\sig_0(x)_t\equiv 1$. The signature $\sig_N(x)_t$ lives in a strict subspace $\mathbb{G}_N(\hR^d) \subset T_N(\hR^d)$, known as the {free Carnot group} over $\hR^d$ of step $N$, where $T_N(\hR^d)=\oplus_{k=0}^N(\hR^d)^{\otimes k}$ is the truncated tensor algebra over $\hR^d.$ Furthermore, the exponential map defines the diffeomorphism from the Lie algebra $\mathfrak{g}_N(\mathbb{R}^d)$ to the Lie group $\hG_N(\hR^d)$, namely
\bea\label{sig and log sig}
\hG_N(\hR^d) =\exp ( \mathfrak{g}_N(\mathbb{R}^d)),
\eea 
where $\mathfrak{g}_N(\mathbb{R}^d)$ is the Lie sub-algebra of $T_N(\mathbb{R}^{d})$ generated by the canonical basis $e_i, i=1,\dots,d,$   of $\mathbb{R}^d$, and the Lie bracket  is given by $[a,b]=a\otimes b-b\otimes a$.  
Thus, the $\log$ signature lives in the linear space  $\mathfrak{g}_N(\mathbb{R}^d)$, and we denote logarithm of the signature of the path $x$ as $\mathsf{LS}(x).$ Let $\pi_m(\cdot)$ be the projection map of the signature and the $\log$ signature at order $m$. We denote $\LS_m(X)=\pi_m(\LS(x))$ as the truncated log signature of a path $x$ of order $m$. We introduce the following standard treatment when computing the signature of a path together with the time parameter.
\begin{defn}\label{defn: time aug path}
Given a path $x: [a, b] \rightarrow \hR^d$, we define the corresponding time-augmented path by $\hat{x}_t = (t, x_t)$, which is a path in $\hR^{d + 1}$. 
\end{defn}
We should remark here that a bounded $p$-variation path is essentially determined by its truncated signature at order $\lfloor p \rfloor$ (e.g. \cite{friz2010multidimensional}[Chapter 7]). This means that essentially no information is lost when applying the signature transform of a path at certain order without using the whole signature process. 
\begin{prop}[Universal nonlinearity, \cite{arribas2018derivatives}, see also \cite{kidger2019deep} Proposition A.6]
Let $F$ be a real-valued continuous function on continuous piecewise smooth paths in $\hR^{d}$ and let $\cK$ be a compact set of such paths. Then for all $x \in \cK$ and $\forall \varepsilon$, there exists a linear functional $\cL$ such that,
\bea
|F(x) - \cL(\sig(x))| \leq \varepsilon.
\eea
\end{prop}

We introduce the signature and the $\log$ signature layer in \cite{liao2019learning}.
\begin{defn}[Signature and $\log$ Signature Sequence Layer]
Consider a discrete $d$-dimensional time series $(x_{t_{i}})_{i = 1}^{n}$ over time interval $[0, T]$. A $(\log)$ signature layer of degree $m$ is a mapping from $\hR^{d \times n}$ to $\hR^{\hat{d} \times N}$, which computes $(\sig_{k})_{k = 0}^{N-1}$ $($or $(\LS_{k})_{k = 0}^{N-1})$ as an output for any $x$, where $\sig_k$ $($or $\LS_{k})$ is the truncated $(\log)$ signature of $x$ over time interval $[u_{k}, u_{k+1}]$ of degree $m$ as follows:
\bea
\sig_{k} = \pi_m(\sig_{[u_{k}, u_{k + 1}]}),~(\text{or}~\LS_{k} =\pi_m( \LS_{[u_{k}, u_{k + 1}]} )),
\eea
where $k \in \{0, 1, ..., N -1\}$ and $\hat{d}$ is the dimension of the truncated $(\log)$ signature.
\end{defn}

\subsection{Main algorithms}
In this section, We consider the following Markovian FBSDE,
\bea\label{markov FBSDE}(\mathsf M)
\begin{cases}
X_{t}& = x + \int_{0}^{t}b(s,X_{s})ds + \int_{0}^{t}\sigma(s,X_{s})dW_{s}, \\
Y_{t} &= g(X_{T}) + \int_{t}^{T}f(s, X_{s}, Y_{s}, Z_{s})ds - \int_{t}^{T}Z_{s}dW_{s},
\end{cases}
\eea
and the non-Markovian FBSDE, 
\bea\label{non-Markov FBSDEs}(\mathsf{NM}) 
\begin{cases}
X_{t}& = x + \int_{0}^{t}b(s,X_{s})ds + \int_{0}^{t}\sigma(s,X_{s})dW_{s}, \\
Y_{t} &= g(X_{\cdot\wedge T}) + \int_{t}^{T}f(s, X_{\cdot\wedge s}, Y_{s}, Z_{s})ds - \int_{t}^{T}Z_{s}dW_{s},
\end{cases}
\eea
for $t \in [0,T]$. In both the Markovian ($\mathsf{M}$) and non-Markovian ($\mathsf{NM}$) FBSDEs system above, we denote $\{W_{s}\}_{0\leq s \leq T}$ as $\mathbb R^d$-valued Brownian motion. Throughout the paper, unless otherwise stated, the process $X, Y$, and $Z$ take values in $\hR^{d_1}, \hR^{d_2},$ and $\hR^{d_2 \times d}$, respectively. We denote $g(X_{T})$ as the state dependent terminal condition and denote $g(X_{\cdot\wedge T})$ as the terminal condition depending on the path of $X$, which corresponds to the the payoff function in the option pricing problem. The pair $(Y_{t}, Z_{t})_{0 < t < T}$ solves the BSDE in $(\mathsf M)$ and $(\mathsf{NM})$ respectively.

We present signature/ log-signature FBSDE numerical schemes in detail. We first partition the time horizon $[0, T]$ into $n$ time steps with a mesh size $\Delta t: =T/n$, and the time partition is given by $0 = t_0 < t_{1} < \dots < t_n = T$. The state process $X$ is generated from Euler scheme as
\bea 
X_{t_{i + 1}}^{n}  = X_{t_{i}}^{n} + b(t_{i},X^{n}_{t_{i}}) \Delta t  + \sigma(t_{i},X^{n}_{t_{i}}) \Delta W_{t_{i+1}},
\eea 
where $\Delta W_{t_{i+1}}:= W_{t_{i+1}}- W_{t_{i}}$ denotes the increment of the Brownian motion. Next, for some $k \in \{a\in \mathbb Z^+: n/a \in \mathbb Z^+\}$, we partition the time interval $[0, T]$ into $\tilde{n} := n/k$ segmentations with step size $\Delta u := k \Delta t$. The segmentation can be written as $0 = u_0 < u_1 (= t_{k}) < \cdots < u_{\tilde{n}} = T$. Then we compute the signature/log-signature \footnote{The numerical implementation of the signature/log-signature transformation was borrowed from \cite{iisignature}.} of the forward process $X$ truncated at order $m$ based on the segmentation $(u_i)_{1\le i\le \tilde{n}}$, which is denoted as $(\pi_m(\sig(X^{n})_1), \dots, \pi_m(\sig(X^{n})_{\tilde{n}})$. Moreover, we approximate the process $Z$ using a recurrent neural network (RNN) with truncated signature / log-signature at order $m$ as the inputs. Namely, we denote
\bea\label{sig layer for FBSDE}
Z^{\theta,\sig}_{u_i}:=\mathcal R^{\theta}(\pi_m(\sig(X^{n})_0),\cdots,\pi_m(\sig(X^{n})_{i-1}) ),
\eea 
for $i \in \{1, \dots, \tilde{n}\}$, which is the output of the RNN \footnote{In particular, the recurrent network in this paper is the LSTM network \cite{LSTM}.} with truncated signature of forward process $X$ at order $m$ as the inputs. Similarly, we denote 
\bea \label{log sig layer for BSDE}
Z^{\theta,\LS}_{u_i}:=\mathcal R^{\theta}(\pi_m(\LS(X^{n})_0), \cdots,\pi_m(\LS(X^{n})_{i-1}))
\eea 
as the output of RNN with log signature inputs.
Then, the discrete scheme of $Y$ for the Markovian BSDE is given as below,  
\begin{equation}\label{eqn: Y_sig}
\begin{aligned}
    Y^{\tilde{n},\sig}_{u_i} :=  & Y^{\tilde{n},\sig}_{u_{i-1}} - f(u_{i-1},X^n_{u_{i-1}}, Y^{\tilde{n},\sig}_{u_{i-1}}, Z^{\theta,\sig}_{u_{i-1}}) \Delta u_i + Z^{\theta,\sig}_{u_{i-1}}\Delta W_{u_{i}},\quad \mbox{for}\quad (\mathsf{M}).
\end{aligned}
\end{equation}
Similarly, for the non-Markovian problem, we define 
\begin{equation}\label{eqn: Y_sig path}
\begin{aligned}
    \bar  Y^{\tilde{n},\sig}_{u_i} :=  & \bar Y^{\tilde{n},\sig}_{u_{i-1}} - f(u_{i-1},X^n_{[0,u_{i-1}]}, \bar Y^{\tilde{n},\sig}_{u_{i-1}}, Z^{\theta,\sig}_{u_{i-1}}) \Delta u_i + \bar Z^{\theta,\sig}_{u_{i-1}}\Delta W_{u_{i}},\quad  \mbox{for} \quad (\mathsf{NM}),\\
\end{aligned}
\end{equation}
where $\Delta W_{u_{i+1}}:= W_{u_{i+1}}- W_{u_{i}}$.
 Lastly, the objective is to minimize the loss function $l(\theta, Y_0, Z_0) := \hE[(Y^{\tilde{n}, \sig}_{T} - g(X^n_{T}))^{2}]$ $($or $l(\theta, Y_0, Z_0) := \hE[(Y^{\tilde{n}, \sig}_{T} - g(X^n_{\cdot\wedge T}))^{2}]$ for non-Markovian FBSDE$)$, and update parameters $\theta$ by stochastic gradient descent. The algorithms for log-signature follows similarly by changing the $\sig$ layer with $\LS$ layer in the algorithm. The full algorithm for the $\sig$-layer FBSDE (or $\LS$-layer FBSDE) is presented in
Algorithm \ref{alg:the_alg}.

\begin{algorithm}
\caption{Deep signature/log-signature FBSDE algorithm.}
\label{alg:the_alg}
\begin{algorithmic}[1]
\State Initialize $Y_0, Z_0$. Initialize mesh size $\Delta t$, mini-batch size $M$, total number of paths $\hat{N}$, signature order $m$, number of segments $\tilde{n}$, loss threshold $\varepsilon$.
\State Generate data. (1) Simulate $\hat{N}$ paths of Brownian motions $(W^j_{t_1},\cdots,W^j_{t_n})_{1\le j\le \hat{N}}$ and (2) generate $\hat{N}$ paths of state processes $(X^{j,n}_{t_1},\cdots,X^{j,n}_{t_n})_{1\le j\le \hat{N}}$. (3) Compute signatures of state processes $(\pi_m(\sig(X^{j,n})_0),\cdots,\pi_m(\sig(X^{j,n})_{\tilde{n}-1}))_{1\le j \le \hat{N}}$ \Big(\text{or} $(\pi_m(\LS(X^{j,n})_0),\cdots,\pi_m(\LS(X^{j,n})_{\tilde{n}-1}))_{1\le j \le \hat{N}}$\Big).
    \While{$loss(\theta, Y_{0}, Z_{0})>\varepsilon$}
    \State Randomly select a mini-batch of data, with batch size $M$.
        \State \textbf{for} $i\in \{1,\cdots, \tilde{n}\}$
        \State \quad \quad  $Z^{j,\theta,\sig}_{u_i}=\mathcal{R}^{\theta}(\pi_m(\sig(X^{j,n})_{0}),\cdots,\pi_m(\sig(X^{j,n})_{i-1}) )$.
        \State  $ \Big(\text{or}\quad Z^{j,\theta,\LS}_{u_i}=\mathcal{R}^{\theta}(\pi_m(\LS(X^{j,n})_0) ,\cdots,\pi_m(\LS(X^{j,n})_{i-1}) )\Big).$
        \State \quad \quad 
        Compute $Y^{j,\tilde{n},\sig}_{u_{i+1}}$ 
        from Euler scheme \eqref{eqn: Y_sig}.
        \State \textbf{end} 
        \State Compute $loss(\theta,Y_0,Z_0)=\frac{1}{M} \sum_{j=1}^{M} (Y_{T}^{j, \tilde{n},\sig} - g(X_{T}^{j,n}))^{2}$ for problem ($\mathsf{M}$). 
        
        \quad \quad $loss(\theta,Y_0,Z_0)=\frac{1}{M}\sum_{j=1}^M (Y_{T}^{j,\tilde{n},\sig} - g(X^{j,n}_{\cdot \wedge T}))^{2}$ for problem ($\mathsf{NM}$).
        \State \Big( \text{or} Compute $loss(\theta,Y_0,Z_0)=\frac{1}{M} \sum_{j=1}^{M} (Y_{T}^{j, \tilde{n},\LS} - g(X_{T}^{j,n}))^{2}$ for problem ($\mathsf{M}$). 
        
        \quad \quad $loss(\theta,Y_0,Z_0)=\frac{1}{M}\sum_{j=1}^M (Y_{T}^{j,\tilde{n},\LS} - g(X^{j,n}_{\cdot \wedge T}))^{2}$ for problem ($\mathsf{NM}$).\Big)
        \State Minimize loss, and update $\theta$ by stochastic gradient descent.
    \EndWhile 
\end{algorithmic}
\end{algorithm}
We keep the following standard assumptions on the coefficients for FBSDEs.
\begin{assumption}
\label{FBSDE Assumption}
Let the following assumptions be in force.
\begin{itemize}
\item $b, \sigma, f, g$ are deterministic taking values in $\hR^{d_1},~ \hR^{d_1 \times d},~ \hR^{d_2},~ \hR^{d_2}$, respectively; and $b(\cdot, 0), \sigma(\cdot, 0), f(\cdot, 0, 0, 0)$ and $g(0)$ are bounded.
\item $b, \sigma, f, g$ are $C^k$-smooth  with respect to all variables $(t,x, y, z)$ for any desired $k\in\mathbb N_+$ and all derivatives are bounded by constant $L$.
\end{itemize}
\end{assumption}
We are now ready to present the universality approximation property of deep signature/log-signature Markovian FBSDE.
\begin{lem}
\label{Zsig approx}
Let \textbf{Assumption \ref{FBSDE Assumption}} be in force. Assume that $kh<\delta$ for any small $\delta>0$, for any given $T>0$, for some constant $C>0$ depnding on $T$ and $L$ in \textbf{Assumption \ref{FBSDE Assumption}}, and for any $\varepsilon>0$, there exists recurrent neural network $\mathcal R^{\theta}$, such that
\bea
\sum_{i = 0}^{\tilde{n}-1}\hE\Big[\int^{u_{i + 1}}_{u_{i}}|Z_{t} - Z_{u_{i}}^{\theta, \sig} |^{2}dt\Big] \leq C[1 + |x|^{2}]\delta+\varepsilon. \notag
\eea
\end{lem}
Furthermore, we have the following estimate.
\begin{thm}\label{main thm}
Let \textbf{Assumption \ref{FBSDE Assumption}} be in force. Assume that $kh<\delta$ for any small $\delta>0$, for any given $T>0$, for some constant $C>0$ depnding on $T$ and $L$ in \textbf{Assumption \ref{FBSDE Assumption}}, and for any $\varepsilon>0$, there exists recurrent neural network $\mathcal R^{\theta}$, such that
\bea
\max_{0 \leq i \leq \tilde{n}}\hE[\sup_{u_{i} \leq t \leq u_{i+1}}|Y_{t} - Y^{\tilde  n,\sig}_{u_{i}}|^{2}] \leq C[1+|x|^2+\varepsilon]\delta. \notag
\eea
\end{thm}
The same estimates follow after replacing  $Z_{\cdot}^{\theta, \sig}$ with $Z_{\cdot}^{\theta, \LS}$ in Lemma \ref{Zsig approx} and Theorem \ref{main thm}. We proved Lemma \ref{Zsig approx} and Theorem \ref{main thm} in Section \ref{convergence Markov}.   Similar results hold true for non-Markovian FBSDE as well, we postpone the analysis in Section \ref{convergence analysis NM FBSDE} 

\section{Numerical results}
\label{numerics}
In this section, we implement our algorithm to a wide range of applications including European call option, lookback option under Black-Scholes model, European call option under Heston model, and a high dimensional example etc. \footnote{The code could be found in the following URL link: \url{ https://github.com/zhaoyu-zhang/Sig-logSig-FBSDE}. The  desktop we used in this study is equipped with an i7-8700 CPU and a RTX 2080Ti GPU. For all the examples in this paper, we generated in total of $\hat{N}=100,000$ paths for the forward processes. $1,000$ paths were used to test, and the rest were used to train the neural network.}. In summary, our $\sig$/$\LS$-FBSDE method has the following advantages over other numerical methods in the current literature:
\begin{enumerate}
    \item Our algorithm is capable to find a more accurate solution to the FBSDE. 
    \item Our algorithm is capable to approximate the true solution efficiently in terms of computation time.
    \item Our algorithm is capable to handle high frequency data in a long time duration. The results are accurate and computation times are efficient.
    \item Our algorithm is capable to handle high dimensional and non-linear scenarios. 
\end{enumerate}
Throughout this section, we denote $(\Omega, \cF, (\cF_t)_{t \geq 0}, \hP)$ as the filtered probability space and denote $\hQ$ as the risk neutral measure. 

\subsection{Best Ask Price for GBM European Call Option}

The limit order book spread has been extensively investigated through \emph{No Arbitrage Bound}/\emph{No Good Deal bound} in incomplete markets \cite{carr2001pricing, cherny2009new,cochrane2000beyond, hansen1991implications, madan2010illiquid, merton1973theory}. Traditionally, under a risk neutral measure $\hQ$, one may assume the underlying asset $X$ follows a geometric Brownian motion i.e., 
\bea
dX_t = X_t(r_t dt + \sigma_t dW_t), \q X_0 = x_0, \notag
\eea
where $W_t$ is a standard Brownian motion under $\hQ$. For implementation convenience, we usually use constant $r$ and $\sigma$ to represent $r_t$ and $\sigma_t$ in numerical examples. By no good deal theory, the best ask price for the European call option at level $\kappa$ ($\kappa$ can be thought as the bound for girsanov kernels) can be represented as
\bea
\label{repb}
\mathrm P^{ask, \kappa} =\sup_{\hP \in \cQ^{ngd, \kappa}}\hE^{\hP}[\b_0^T(X_{T} - K)^+],
\eea
 where the set $\cQ^{ngd, \kappa}$ is nonempty and called the no good deal pricing set at level $\kappa$, $\beta_t^T$ is the discount factor defined by fixed risk-free interest rate $r_t$ . More details can be referred to \cite{bielecki2013dynamic, cherny2009new, Man2021,madan2010illiquid}. In our setting, we define $\cQ^{ngd, \kappa}$ as follow
\bea\label{abr}
\cQ^{ngd, \kappa} := \Big\{\hQ^{\th}:  \frac{d\hQ^{\th}}{d\hQ} = M_T(\lambda^\th, W); \sup_{t\in[0,T]} ||\lambda_t^\th|| \leq \kappa \Big\},
\eea 
where $M_T(\lambda^{\theta}, W) := \exp\{\int_0^T \lambda^{\theta}_t dW_t -\frac{1}{2}\int_0^T ||\lambda^{\theta}_t||^2 dt\}$, the process $\lambda_t^{\theta}$ denote all possible girsanov kernels and their bound is $\kappa$.

\begin{remark} There are several notions to introduce the kernel function $\lambda_t^{\theta}$. The kernel ambiguity introduced by drift uncertainty with discount factor $\beta_0^T(r^\theta)$ has been recently studied in \cite{Man2021}. Here we work in a simplified version, where we consider the discount factor fixed, we define 
$
\lambda_t^{\theta}:=\sigma_t^{-1}(\alpha_t^{\theta}-r_t\textbf{1})$, 
where $\alpha_t^{\theta}\in\Lambda$,
and 
$\Lambda=\{x\in\mathbb R^d, (x-r_t\textbf{1})^{\ts}\sigma_t\sigma_t^{\ts}(x-r_t\textbf{1})\le \kappa^2, \forall t \in [0, T] \}$. The main motivation to consider a fixed interest rate $r_t$ instead of $r_t^{\theta}$ is that we could numerically compute the lower and upper bound by using the empirical calibration of $r_t$ from the market data.
\end{remark}
With the specification of the pricing measure set $\cQ^{ngd, \kappa}$ in \eqref{abr}, we can show that \eqref{repb} is closely linked to the following BSDE. The proof follows from the comparison theorem for BSDEs, and we refer details in \cite{peng2004filtration}.

\begin{thm}
\label{BSDErep}
Assume no good deal assumption, then one can obtain that the best ask/bid price \eqref{repb} at level $\kappa$ are unique solutions to the following BSDEs when $t = 0$
\bea
\label{GBM-BSDE}
Y_{t}^{\pm} = \xi_T + \int_{t}^{T}\mathrm G^{\pm}(X_{t}, Y_t, Z_{t})dt - \int_{t}^{T}Z_{t}dW_{t}, 
\eea 
where $\xi_T := (X_T - K)^+$ and $\mathrm G^{\pm}$ are optimised drivers over all possible kernels in \eqref{abr},
\bea
\mathrm G^{-}(x, y,  z) = \min_{||\l^\th|| \leq \kappa} f(x, y, z, \l^\th) \q \textit{and} \q \mathrm G^{+}(x, y,  z) =\max_{||\l^\th|| \leq \kappa} f(x, y, z, \l^\th). \notag
\eea
Here we define $f(x, y, z, \l^\theta) := -ry + z\l^\theta$. Furthermore, we can obtain the optimised drivers as follow
\bea
\mathrm G^{\pm}(x, y,  z) = \pm \kappa ||z||- ry. \notag
\eea
\end{thm}

In this example, we implement 1-dimensional best ask scenario for \eqref{GBM-BSDE}, and we compare results from our signature methods with simple neural network method. We choose the following parameters for the simulation $x_0 = 100, \sigma = 0.20,  r = 0.05, \kappa = 0.05, K = 80, T = 1, m = 3, \tilde{n} = 5$, and batch size 1000. 
\begin{center}
\captionof{table}{Best ask price for GBM European call option \label{Tab:vanila}}
\resizebox{\textwidth}{!}{
\begin{tabular}{ |c|c|c|c|c|c| } 
 \hline
Simple NN &  Sig-LSTM  &  Sig-LSTM   & Sig-LSTM & Sig-LSTM  \\ 
 $n=100$ &  $\tilde{n}=5$, $n=100$ & $\tilde{n}=5$, $n=500$ & $\tilde{n}=5$, $n=1000$ & $\tilde{n}=5$, $n=5000$ \\ \hline
 25.526 &  25.48 & 25.46 &  25.46 & 25.45\\ 
 \hline
\end{tabular}}
\end{center}

As we see from Table \ref{Tab:vanila}, our algorithm combining signature with LSTM neural network (labeled as Sig-LSTM) outperforms the simple neural network method in terms of efficiency, our algorithm runs 20 times faster than simple neural network approach with $n=100$. This is what we should expect, since for each iteration our algorithm runs 5 steps segmented by signature ($\tilde{n} = 5$) instead of 100 steps ($n=100$) in the simple neural network approach with Euler scheme. Also, as we can see in Table \ref{Tab:vanila}, the result converges to $25.45$ when $n$ increases. More accuracy and time efficiency results comparisons are illustrated in the lookback option example, which is a path dependent option.

\subsection{Lookback Option Example}
In this example, we consider the classical Black-Scholes model setting.  Under the risk neutral measure $\hQ$, the stock prices $(X_t)_{t \geq 0}$ follows a geometric Brownian Motion with constant interest rate $r$, and volatility $\sigma$,
\beaa
dX_t = r X_t dt + \sigma X_t dW_t, \quad X_0 = x_0.
\eeaa 
Lookback option is one of the path-dependent financial derivatives. A lookback call option with floating strike is given by the payoff function
\beaa g(X_{[0,T]})  = X_T - \inf_{0 \leq t \leq T} X_t. 
\eeaa 
It is clear that the option price $Y_t$ has the form
\beaa 
Y_t = e^{-r(T-t)} \hE^{\hQ} [g(X_{[0, T]}) | \cF_t].
\eeaa 
Fortunately, $Y_t$ has an explicit solution, (e.g. \cite{musiela_rutkowski_2005}),
\beaa 
Y_t = X_t \Phi(a_1) - m_t e^{-r(T-t)} \Phi(a_2) - X_t \frac{\sigma^2}{2r} \left( \Phi(-a_1) - e^{-r(T-t)} \left(\frac{m_t}{y_t}\right)^{2r/\sigma^2} \Phi(-a_3)\right),
\eeaa 
where $m_t := \inf_{0 \leq u \leq t} X_t$, and 
\beaa 
a_1 = \frac{\log(X_t/m_t) + (r + \sigma^2/2)(T-t)}{\sigma \sqrt{T-t}}, \quad a_2 = a_1 - \sigma \sqrt{T-t} \mbox{ and } a_3 = a_1 - \frac{2r}{\sigma} \sqrt{T-t}.
\eeaa 
In the meantime, the option price $Y_t$ can also be represented as  a solution to the following BSDE,
\begin{equation*}
    \begin{cases}
    	 dY_t = & rY_tdt + Z_t dW_t, \\
    Y_T = & X_T - \inf_{0 \leq t \leq T} X_t.    \end{cases}
\end{equation*}
Therefore, we are able to apply our numerical method, and compare solutions with the true solution, and solutions from other numerical schemes. 

In this example, we choose the following parameters in simulation, $x_0 = 1, \sigma = 1, \ r = 0.01, T = 1, m=3$. In Figure \ref{fig:lookback_T1}, we compare the convergence of lookback option prices from different methods, and different time discretization steps.  Vanilla-LSTM refers to the algorithm that the inputs to the neural networks are the stock prices. PDGM from \cite{saporito2020pdgm} is a numerical scheme based on recurrent neural network, and it is used to solve PPDEs. LogSig-LSTM and Sig-LSTM refer to the two numerical algorithms proposed in this study.  
Figure \ref{fig:lookback_error}  list all computation errors  over different methods and time steps respectively.

 \begin{figure}[htbp]
 \begin{center}
\includegraphics[scale=0.5]{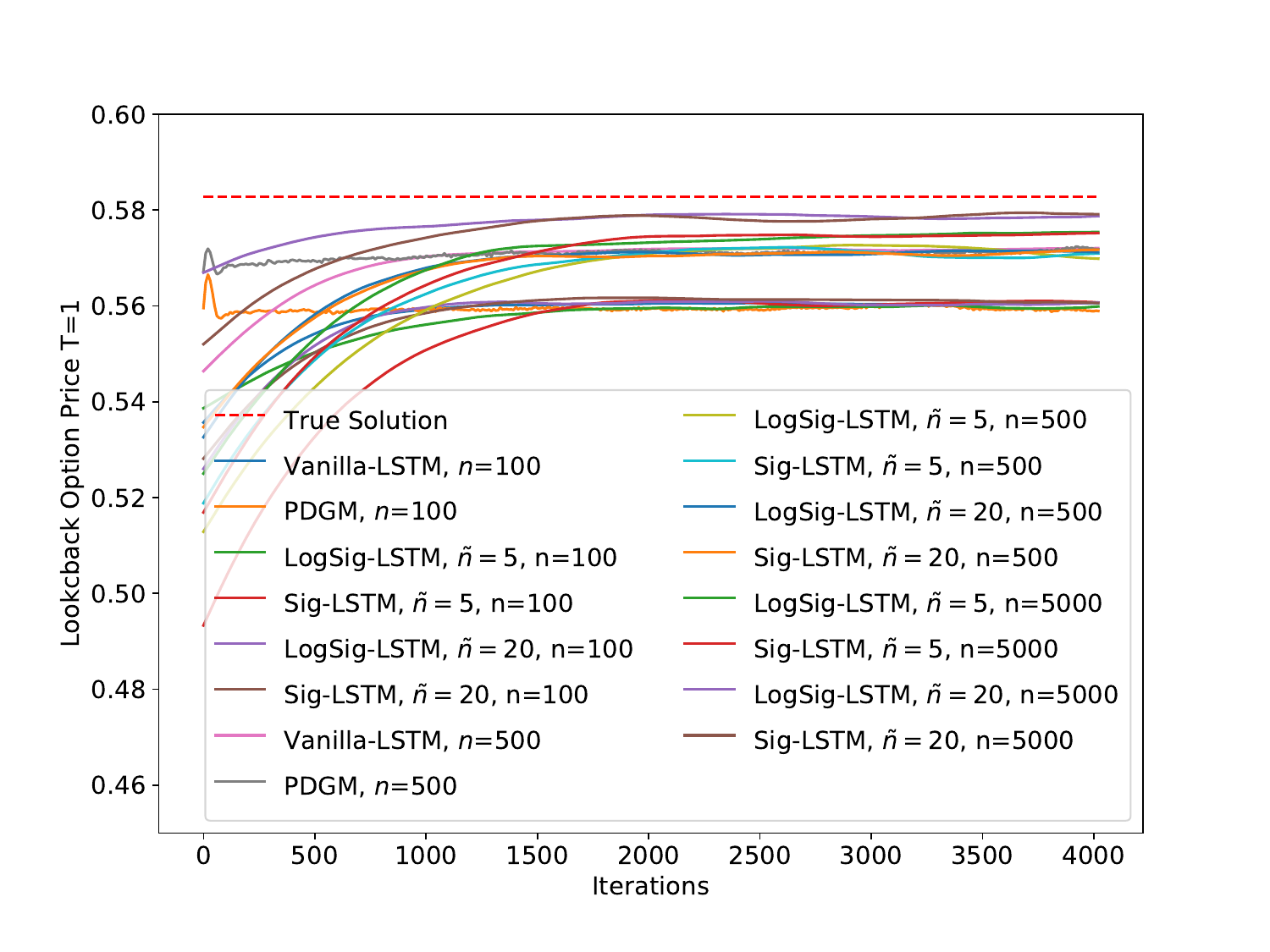}
  \caption{Convergence on lookback option prices ($T=1$) via different methods.}
  \label{fig:lookback_T1}
  \end{center}
 \end{figure}

\begin{figure}[htbp]
 \begin{center}
  \includegraphics[scale=0.5]{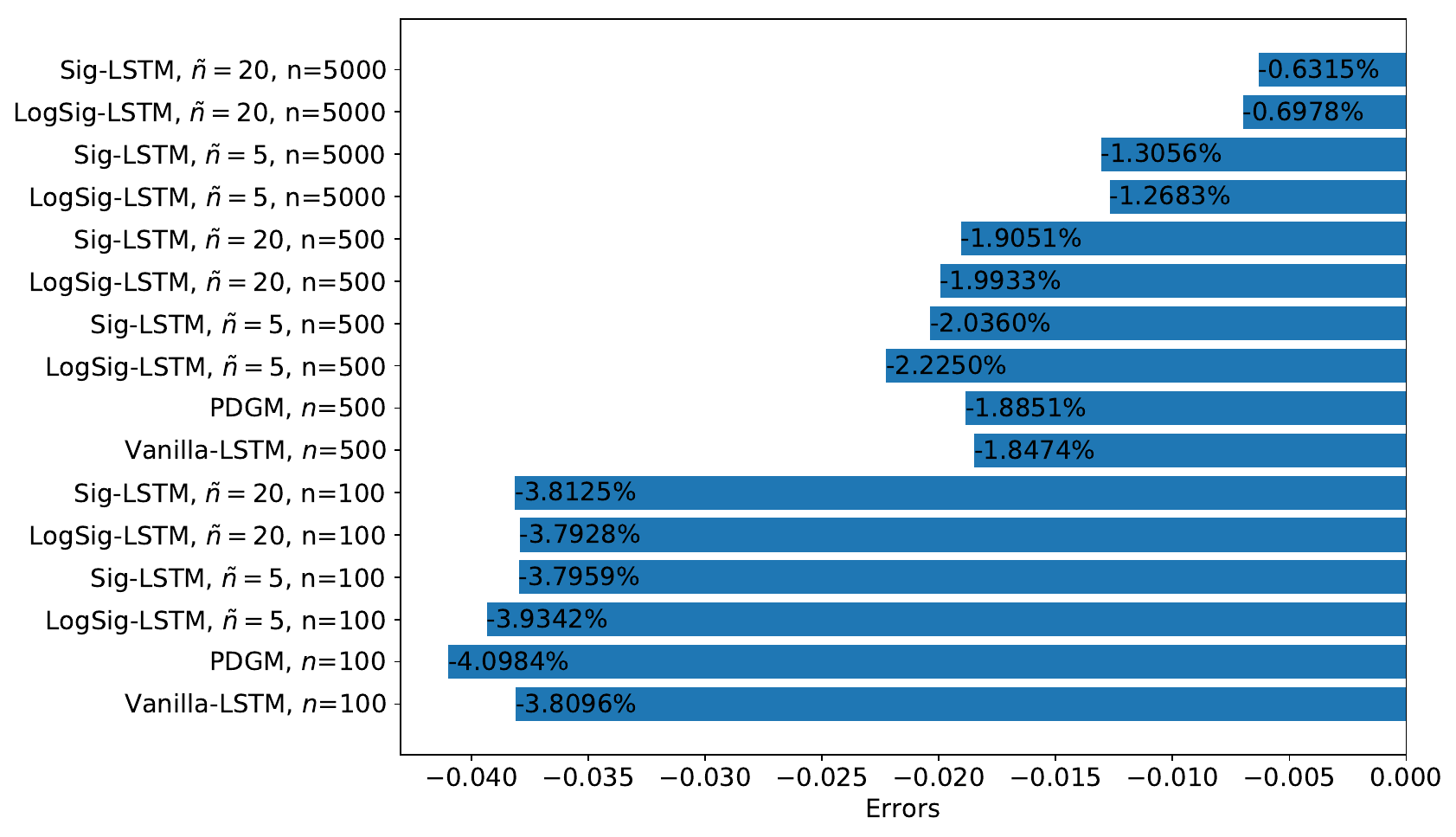}
  \caption{
  Option pricing errors across different methods and time steps.}
  \label{fig:lookback_error}
   \end{center}
 \end{figure}

The first observation is that under the same number of time steps, the numerical solutions from all methods are very similar. Secondly, the key to improve   the numerical solutions to be closer to the true solution is the number of the time steps during simulation, which is quite intuitive. As we can see in Figure \ref{fig:lookback_error}, the numerical error goes down with smaller the mesh sizes.  In particular, with $n=5000$, our log-signature and signature perform the best, and with $\tilde{n} = 20$, the numerical solution is only approximately 0.6\% apart from the true solution. The third observation is that the convergence rate is slower with smaller number of segmentations $\tilde{n}$ in log-signature and signature method. In addition, with a larger number of segmentations, the numerical results are generally better. Therefore, one may be encouraged to have $n$ become as large as possible. However, this is not feasible in practice due to the running times.

\begin{figure}[htbp]
 \begin{center}
  \includegraphics[scale=0.5]{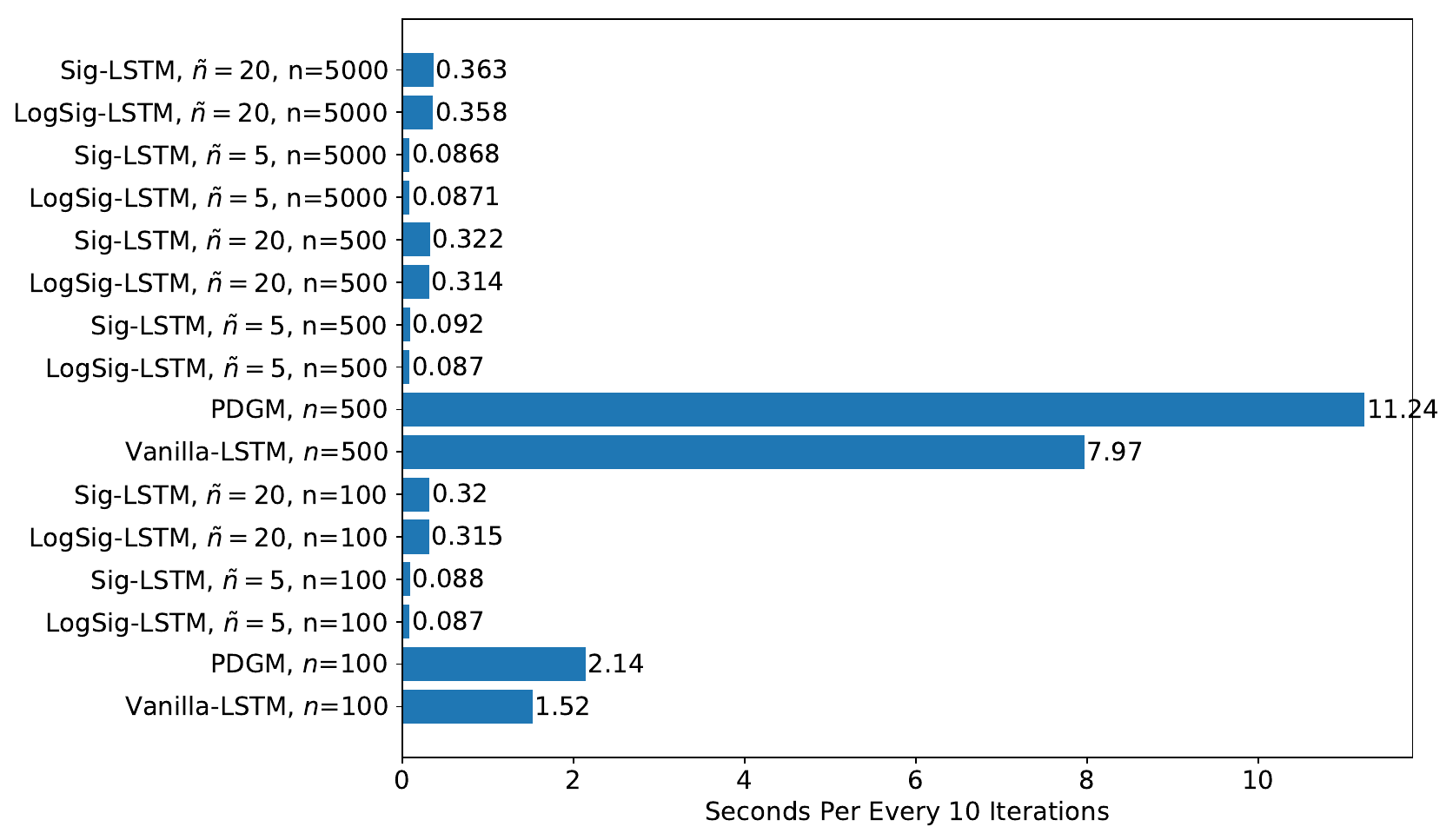}
   \end{center}
 \caption{Computation times over different methods and time steps.}
  \label{fig:lookback_time}
\end{figure}

Figure \ref{fig:lookback_time}  compares the running times over different methods and time steps respectively. The running times are approximately linear with the number of segmentations and time steps. Log-signature and signature methods run 100 times faster with 5 segmentation ($\tilde{n} =5$) than vanilla-LSTM with 500 times steps ($n = 500$). Therefore, summarizing the stock data paths into signature into a few segmentations, and then inputting them into the neural network would save us a great amount of time, and obtain the similar accuracy. 

 In addition, our method can handle high-frequency data. It would be impracticable  to input a stock paths with $n=5000$ into the vanilla-LSTM since it would take too long to train. However, we could first divide the 5000 time steps into 5 or 20 segmentation, and then compute the log-signature and signature of segmentations, which will be finally input into the neural networks. As we can see from this example, our method reaches a higher accuracy in an time efficient manner.

Furthermore, our method could handle high frequency data with a long time duration. In general, for a given time horizon $T$, we could choose $n$ (e.g. n=5000 or larger) large enough such that we can still simulate the asset dynamic with small step size, i.e. $\Delta t$ small, while the number of segments remain fixed. In our algorithm, the time discretization $n$ will only affect the data generation process which is offline. The computation efficiency of our algorihtm is only affected by the number of segments $\widetilde n$. Continuing with lookback option example, now we choose the parameters to be $x_0 = 1, r = 0.01, \sigma = 0.05, T = 10$. Since the numerical difference between log-signature and signature methods are minimal, we only make a comparison between vanilla-LSTM and Sig-LSTM in Figure \ref{fig:lookback_T10}. Figure \ref{fig:lookback_T10_close} plots a closeup of lookback option prices with different time-steps. Comparing to Vanilla-LSTM with $n = 500$, our Sig-LSTM methods with $\tilde{n} =5 $ and $n = 5000$ improves the accuracy by 1.36\%, and underestimates the solution only 0.624\%. In the meantime, our Sig-LSTM method with $\tilde{n} =5 $ and $n = 5000$ runs 100 times faster than Vanilla-LSTM with $n = 500$. 

\begin{figure}[htbp]
 \begin{center}
  \includegraphics[scale=0.4]{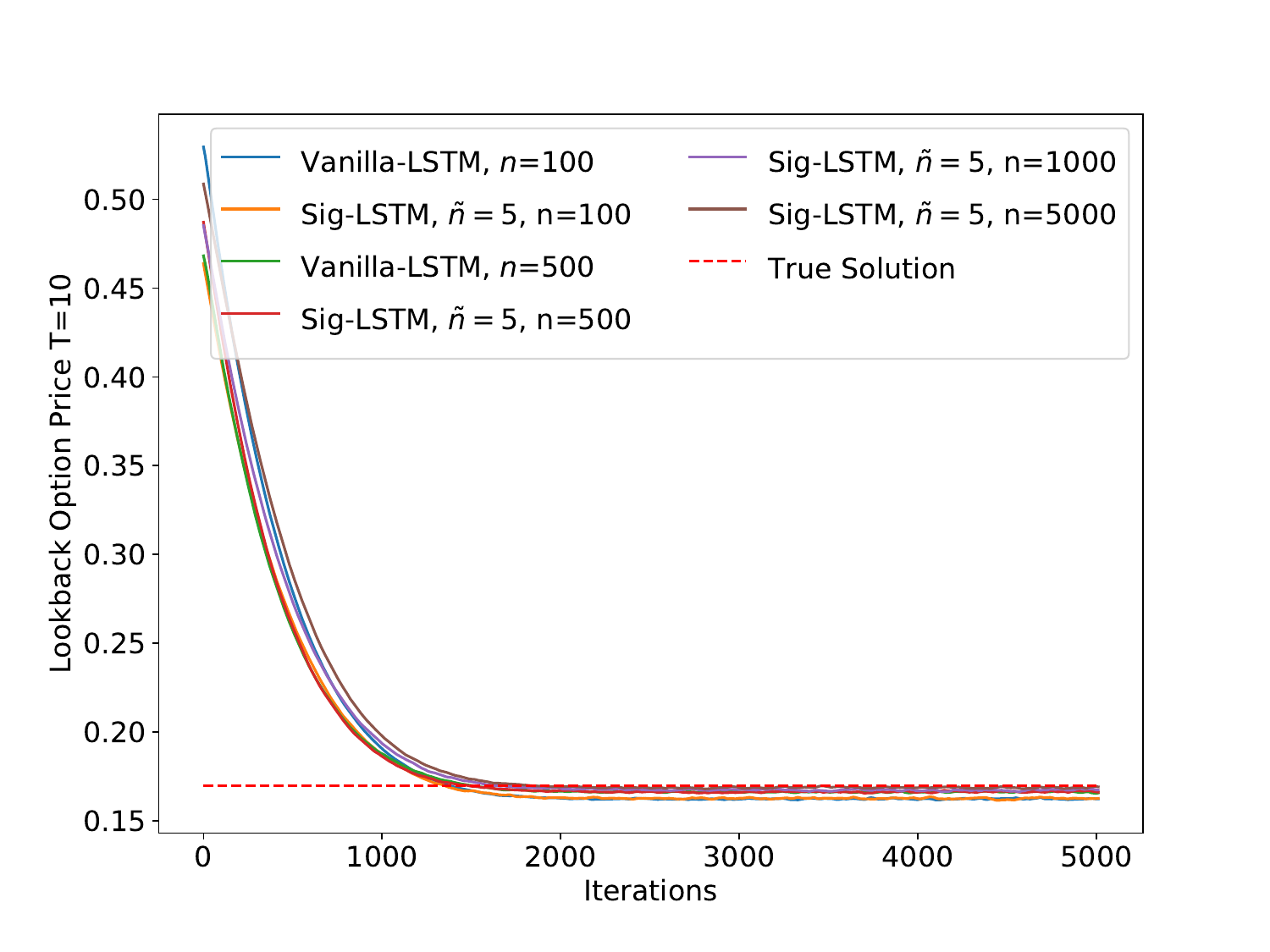}
  \caption{
  High frequency long duration lookback option pricing example ($T=10$).}
  \label{fig:lookback_T10}
  \end{center}
 \end{figure}
\begin{figure}[htbp]
  \begin{center}
  \includegraphics[scale=0.4]{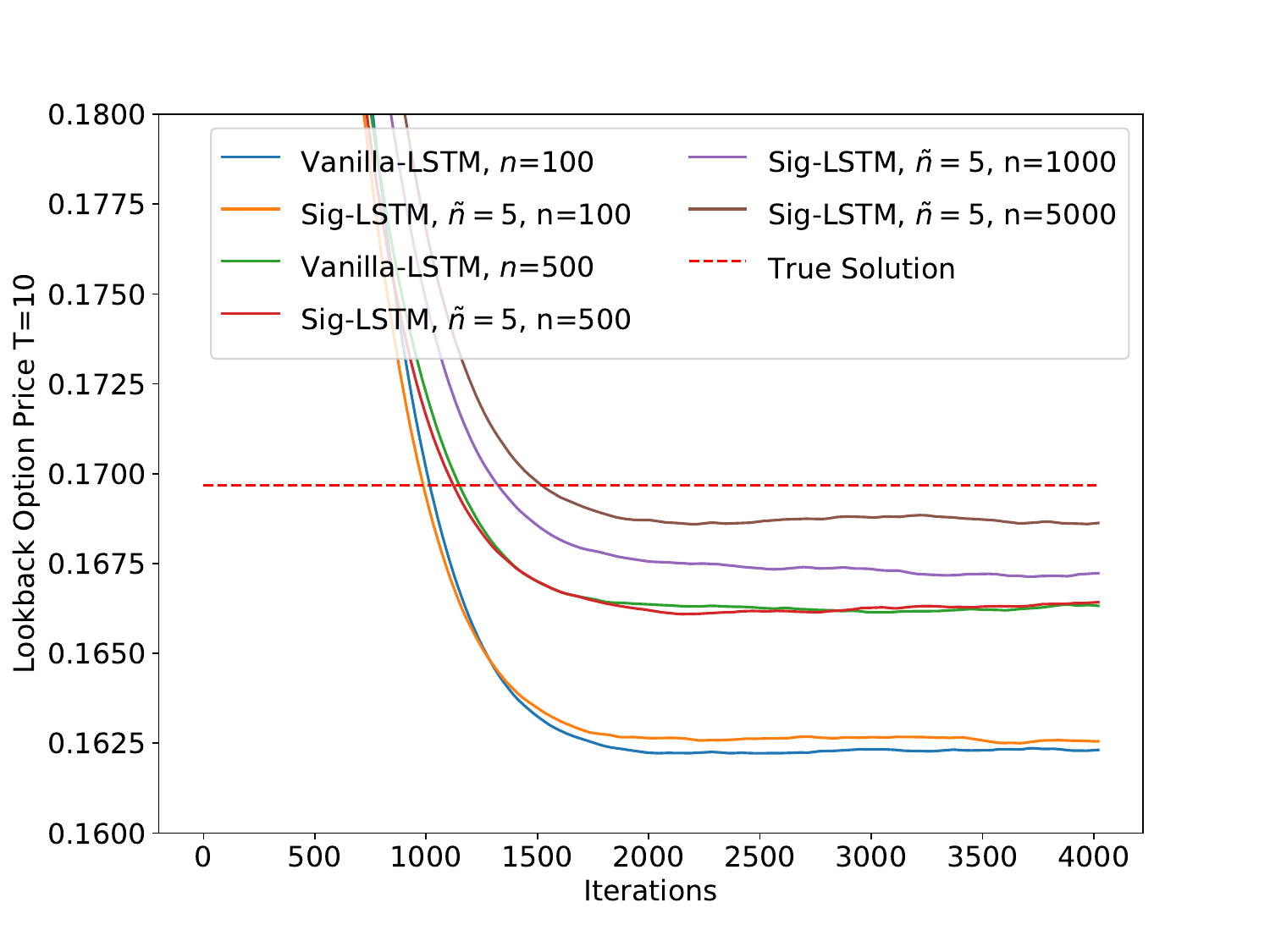}
   \caption{High frequency long duration lookback option pricing example ($T=10$ zoomed plot).}
  \label{fig:lookback_T10_close}
     \end{center}
 \end{figure}


\subsection{European Call Option in the Heston Model under Parameter Uncertainty }


We consider the Heston model in \cite{cohen2017european} for a European call option pricing problem with stochastic volatility model under parameter uncertainty. For $t \in [0, T]$, the asset price $S$ and forward variance process $V$ follows,
\begin{equation*}
\begin{cases}
     dS_t = rS_t dt + \sqrt{V_t}S_t (\rho dW_t^1 + \sqrt{1-\rho^2}dW_t^2),\\
dV_t = (\kappa\theta -[\kappa + \sigma \lambda]V_t)dt + \sigma \sqrt{V_t}dW_t^1,
\end{cases}
\end{equation*}
 and $W^1, W^2$ are two Brownian motions under the risk neutral measure $\hQ$ with correlation $\rho \in (-1, 1)$. Parameters ($\kappa, \theta, \sigma$) are assumed to be nonnegative and satisfy the Feller's condition $2\kappa\theta \geq \sigma^2$ to guarantee that variance process $V$ is bounded below from zero. Moreover, under the parameter uncertainty situation, an elliptical uncertainty set for parameters ($r, \kappa, \beta$, where $\beta \equiv \kappa\theta$) with $(1-\alpha)$ confidence is given by the quadratic form $U = \{u: u^T\Sigma^{-1}_{r,\kappa,\beta}u \leq \chi\}$, where $u$ is the perspective deviance towards the true parameters denoted as $u_t = (r_t -r, \kappa_t - \kappa, \beta_t - \beta)$, $\Sigma^{-1}_{r,\kappa,\beta}$ is the covariance matrix of the parameters and $\chi := \chi^2_3(1-\alpha)$ is the quantile of the chi-square distribution with three degrees of freedom. We should remark here that ellipsoidal specifications of uncertainty appear naturally in multivariate Gaussian settings for the uncertainty about the drifts of tradeable asset prices, and literature can be referred to \cite{dirk2017good, biagini2017robust,garlappi2007porfolio}. 
In \cite{cohen2017european}, the pricing bound for the Heston call option under model ambiguity is derived and proved to be the unique solutions of the following BSDEs with payoff $g(S_T) = (S_T - K)^+$ at maturity $T$, 
\begin{equation}\label{eqn:BSDE_bid_ask_2}
    \begin{cases}
    	 dY_t^{\pm} = -H^{\pm}(S_t, V_t, Y_t^{\pm}, Z_t)dt + Z_t dW_t, \q Y_T^{\pm} = g(S_T), \\
    H^{\pm}(S_t, V_t, Y_t^{\pm}, Z_t) = \pm \sqrt{\chi \eta^T_t \Sigma^T \eta_t} -r Y_t,    
    \end{cases}
\end{equation}
where $\eta_t\in\hR^{3\times 1}$ is the vector of coefficients to the parameter deviances of equation given by
\bea
\eta_t = \left[(\frac{Z_t^2}{\sqrt{1 - \rho^2}\sqrt{V_t}} - Y_t), (\frac{-Z_t^1 \sqrt{V_t}}{\sigma} + \frac{\rho Z_t^2\sqrt{V_t}}{\sigma\sqrt{1 - \rho^2}}), (\frac{Z_t^1}{\sigma \sqrt{V_t}} - \frac{\rho Z_{t}^2}{\sigma \sqrt{1-\rho^2}\sqrt{V_t}}) \right]^T,\notag
\eea
and $Z_t, W_t\in \hR^2.$ Also, the perspective deviance towards the true parameter corresponding to \eqref{eqn:BSDE_bid_ask_2} are $u^{\pm}(S_t, V_t, Y_t^{\pm}, Z_t) = \pm \sqrt{\frac{\chi}{\eta^T_t \Sigma^T \eta_t}}\Sigma\eta_t$. Following the idea in \cite{cohen2017european}, the forward component $X = (S, V)$ of the SDE is generated by standard Euler-Maruyama scheme for the log-price and an implicit Milstein scheme for the variance
\begin{equation*}
    \begin{cases}
     \log S_{t}^\pi = \log S_{t_{i-1}}^{\pi} + (r - \frac{1}{2}V_{t_{i-1}}^{\pi})\Delta_{i} + \sqrt{V_{t_{i-1}}^{\pi}}(\rho \Delta W^1_{t_{i}} + \sqrt{1 - \rho^2}\Delta W_{t_{i-1}}^2), \\   
     V_{t_{i}}^\pi = \frac{V_{t_{i-1}}^\pi + \kappa \theta \Delta_i + \sigma \sqrt{V_{t_{i-1}}^\pi} \Delta W^2_{t_i} + \frac{1}{4}\sigma^2 ((\Delta W^2_{t_{i}})^2 - \Delta_i)}{1 + \tilde{\kappa}\Delta_i},
    \end{cases}
\end{equation*}
where $\Delta W^1_{t_i}, \Delta W^2_{t_i}$ are independent variables generated from the zero-mean normal distribution with variance $\Delta_i$. 


We implement the example in \cite{cohen2017european} with the same experiment set up: $S_0=100,~ V_0=0.0457,~r=0.05,~\kappa=5.070,~\theta=0.0457,~\sigma=0.4800,~\rho=-0.767,~K=100,~T=1$, and covariance matrix $\Sigma=\mathrm{Diag}(2.5e-05,0.25,1e-04)\in\hR^{3\times 3}$.


\begin{figure}[htbp]
 \begin{center}
  \includegraphics[scale=0.45]{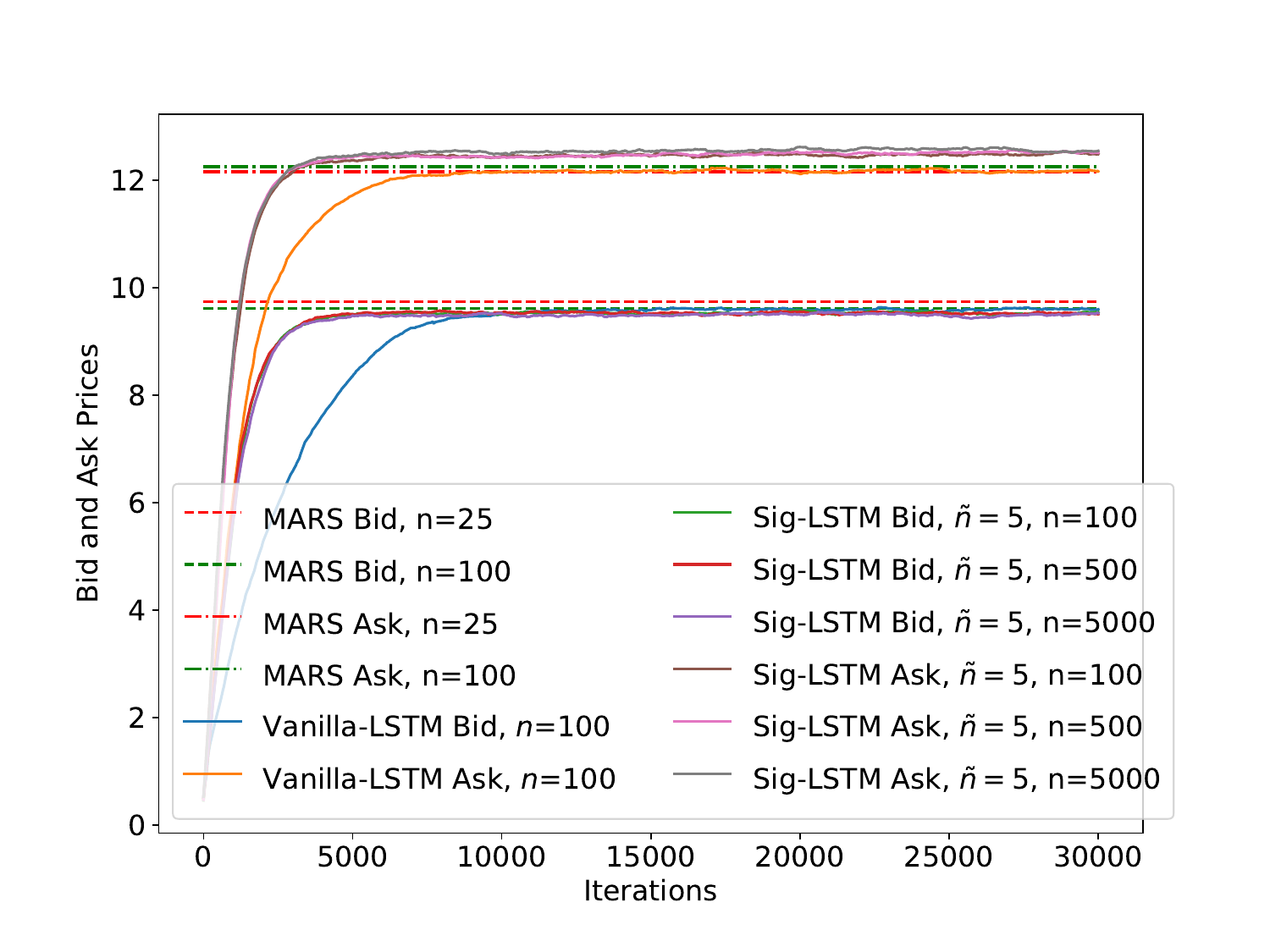}
   \end{center}
\caption{A better bound for bid and ask prices.}
  \label{fig:bid_ask}
 \end{figure}
 
 \begin{figure}[htbp]
\begin{center}
  \includegraphics[scale=0.45]{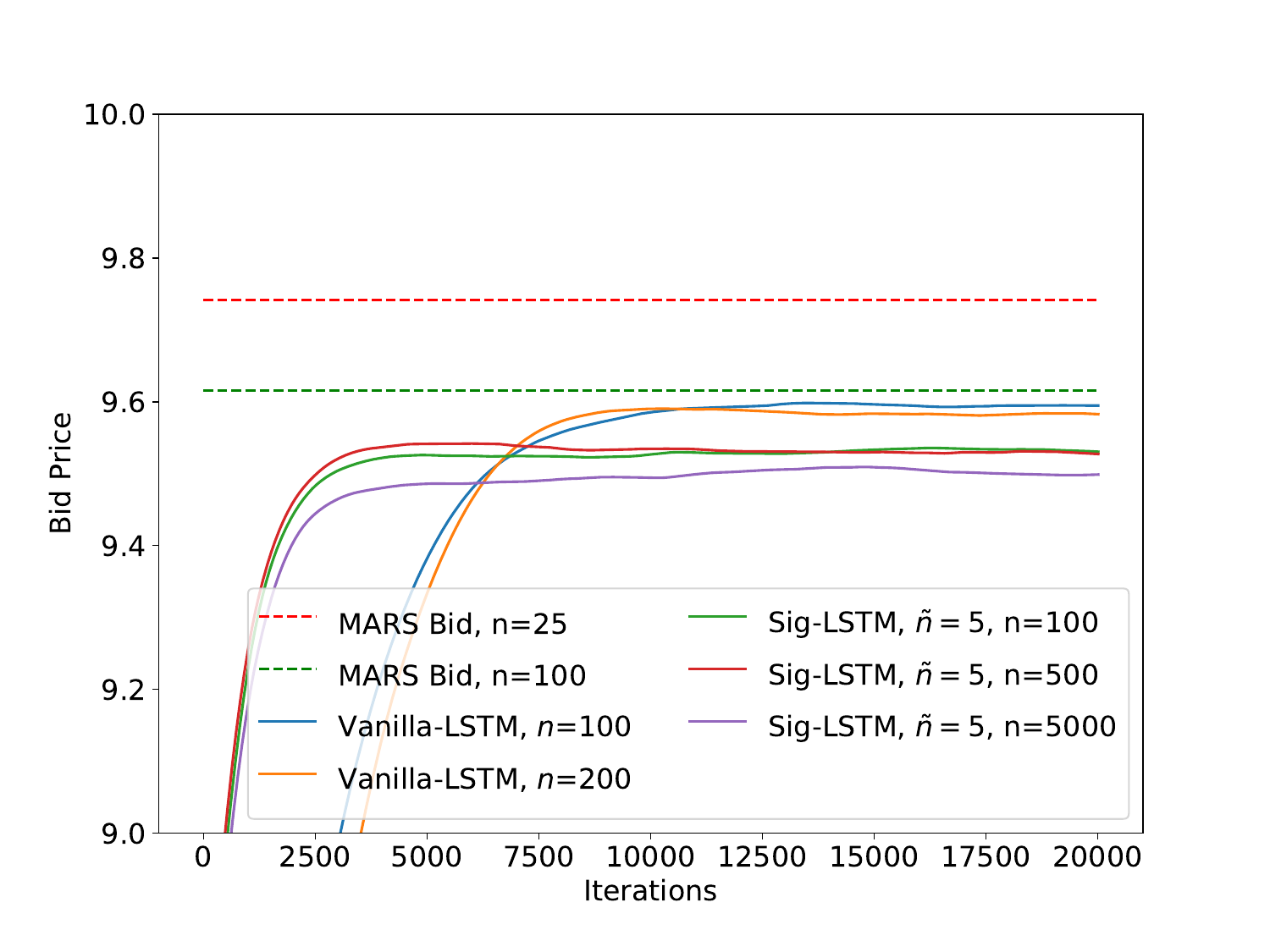}
    \caption{Zoomed plots for bid prices.}
      \label{fig:bid_closeup}
    \end{center}
\end{figure}

\begin{figure}[htbp]
\begin{center}
    \includegraphics[scale=0.45]{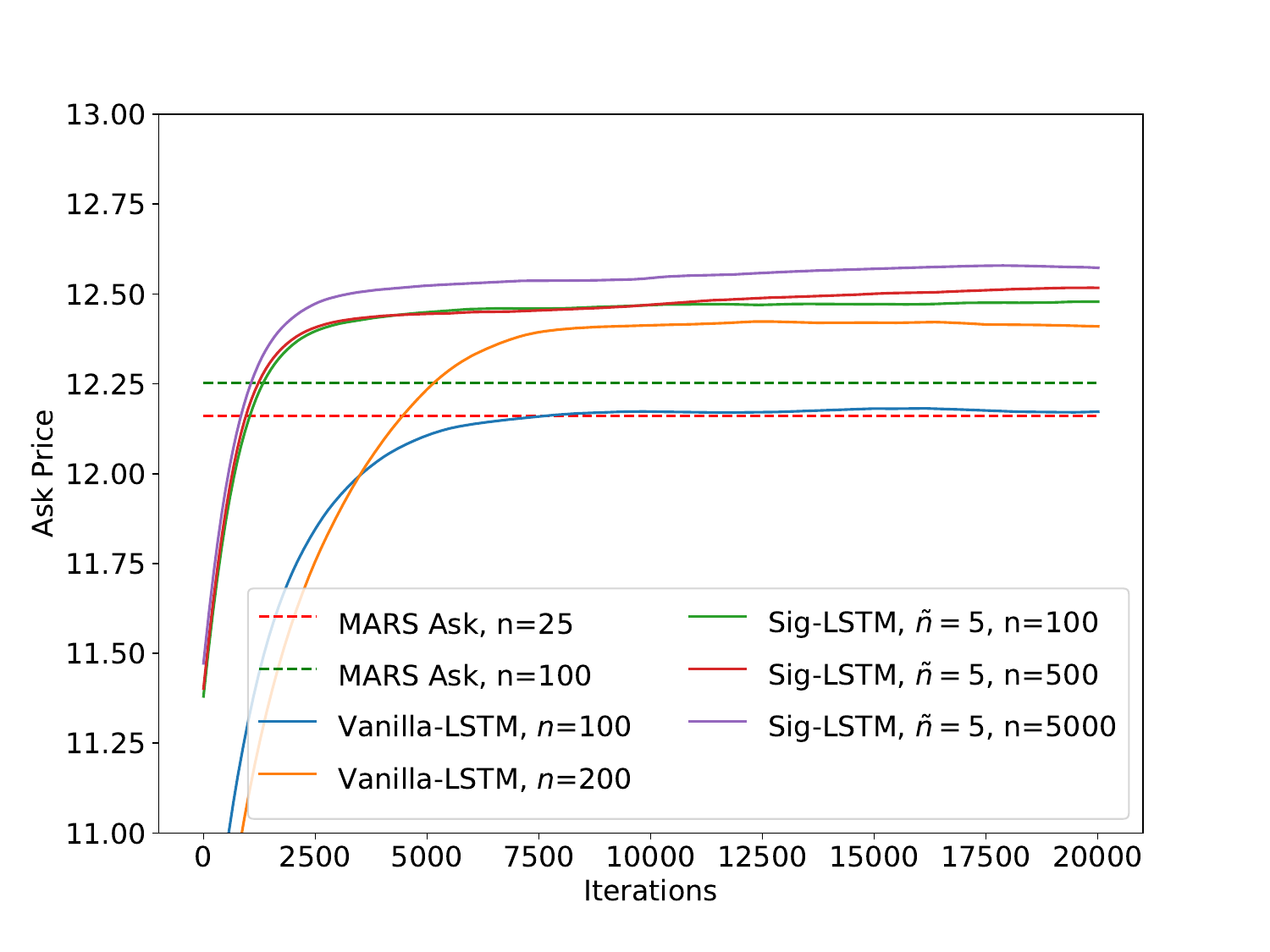} 
    \caption{Zoomed plots for ask prices.}
  \label{fig:ask_closeup}
    \end{center}
\end{figure}
 
From Figure \ref{fig:bid_ask}, we can see that our method provides a better pricing bound over the recursive MARS degree 2 with variance reduction method (denoted as ``MARS" in the figure) in \cite{cohen2017european}, by providing a slightly wider bound for the optimally controlled value process. The zoomed plots are in Figure \ref{fig:bid_closeup} and \ref{fig:ask_closeup}. As we increase the number of time steps to $n = 200$, the Vanilla-LSTM performs better than MARS method with $n=25$, and $n = 100$. Lastly, with $\tilde{n}=5, $ and $n = 5000$, our Sig-LSTM method efficiently improves the bound. This is what we should expect. With a larger number of time discretization $n$, the driver $H^{\pm}$ in \eqref{eqn:BSDE_bid_ask_2} are updated more accurately, which leads to the value process $Y^{\pm}$ in \eqref{eqn:BSDE_bid_ask_2} optimised to a higher degree.
\begin{center}
\captionof{table}{Bid ask prices for European call under Heston model \label{Tab:heston}}
\resizebox{\textwidth}{!}{
\begin{tabular}{ |c|c|c|c|c|c|c| } 
 \hline
 MARS Bid & MARS Bid & Vanilla-LSTM Bid  & Vanilla-LSTM Bid & Sig-LSTM Bid &  Sig-LSTM Bid  & Sig-LSTM Bid \\ 
 $n=25$ & $n=100$ & $n=100$ & $n=200$ & $\tilde{n}=5$, $n=100$ & $\tilde{n}=5$, $n=500$ & $\tilde{n}=5$, $n=5000$ \\ \hline
 9.74 & 9.62 &  9.59 & 9.58 &  9.53 & 9.527 & 9.50\\ 
 \hline\hline
  MARS Ask & MARS Ask & Vanilla-LSTM Ask  & Vanilla-LSTM Ask & Sig-LSTM Ask &  Sig-LSTM Ask  & Sig-LSTM Ask \\ 
 $n=25$ & $n=100$ & $n=100$ & $n=200$ & $\tilde{n}=5$, $n=100$ & $\tilde{n}=5$, $n=500$ & $\tilde{n}=5$, $n=5000$ \\ \hline
 12.16 & 12.25 &  12.17 & 12.41 &  12.48 & 12.52 & 12.57\\ 
 \hline
\end{tabular}}
\end{center}


\subsection{A High Dimensional Example}\label{section high and nonlinear}


In this section, we consider the following path-dependent BSDE, 
\bea \label{high dim BSDE}
Y_t = g(X_{[0,T]})+ \int_t^T f(s, X_{\cdot} ,Y_s, Z_s) ds - \int_t^T Z_s dB_s,
\eea 
and the forward process is given by $dX_t=dB_t, \ X_0 = 0$. Based on the association with PPDE and the nonlinear Feynman-Kac formula, we construct a high dimensional example, which we could find the true solution. For simplicity, we choose  $$f = 0, \quad \text{and}\quad g(X_{\cdot}) = \left( \int_0^T \sum_{i=1}^d X_s^i ds \right)^2.$$
 We then compare the true solution with the solution approximated by our algorithm.  In this example, we use the deep log-signature BSDE algorithm because the input of the network grows exponentially in terms of the dimension. By applying the log-signature layer, we could potentially solve higher dimension problems. With $d = 20, T = 1, \tilde{n } = 5, n = 100$, after 10000 training iteration, the approximated solution of $Y_0$ from our algorithm is 6.60 with an error of 1\% to the true solution of 6.66. Again, our algorithm runs only 5 steps ($\tilde{n} = 5$) during training, which is quite time efficient. Here we remark that even our algorithm is able to approximate the true solution of the high dimensional example, it is more suitable for high frequency, path dependent and long duration data. This is because when generating signatures / log-signatures from high dimensional paths, the dimension of signatures / log-signatures would increase exponentially in terms of the dimension of the path, we could see this from equation \eqref{sig and log sig}. 
 
 We present the equivalent PPDE of our path-dependent BSDE example in \eqref{high dim BSDE}. The setup and definition of path deriatives can be found in Subsection \ref{convergence analysis NM FBSDE}. On the canonical space ($[0,T] \times C([0,T], \dbR^d)$), the PPDE follows
\begin{equation*}
    \begin{cases}
    & \partial_tu + \frac{1}{2} tr(\partial_{\o \o} u) + f(t, \o, u, \partial_\o u) = 0,\\
    & u(T, \o) = \xi(\o).
    \end{cases}
\end{equation*}
For this high dimensional example with generator  $f = 0$, and terminal condition $\xi(\o) = \left( \int_0^T \sum_{i=1}^d \o_s^i ds \right)^2$, the PPDE yields an explicit solution 
$$u(t, \o) = \left( \int_0^t \sum_{i=1}^d \o_s^{i} du\right)^2 + \left(\sum_{i=1}^d \o_t^i \right)^2 (T -t)^2 + 2(T-t) \left(\sum_{i=1}^d \o_t^{i} \right) \int_0^t \sum_{i=1}^d \o_s^{i} ds + \frac{d}{3}(T-t)^3.$$
  

\subsection{Another Nonlinear Example}
In this  example, we apply our algorithm to approximate the solutions of an non-linear FBSDE \eqref{high dim BSDE} with $d=1$, and the generator is
\bea\label{bdse-nonlinear}
	f(t,X_{[0,t]},Y_t,Z_t) &= &-\bigg( \min_{\mu \in [\uu{\mu}, \oo{\mu}]} \mu Z_t + \max(\sin(X_t + \int_0^t X_s ds), 0) (\oo{\mu} + X_t) \\
&&+\min(\sin(X_t + \int_0^t X_s ds), 0) (\uu{\mu} + X_t)
	+ \frac{1}{2} \cos(X_t + \int_0^t X_s ds)\bigg)\nonumber
\eea 
In the numerical implementation, we choose the terminal condition to be $Y_T = \cos(X_T + \int_0^T X_s ds)$, and the forward asset process $dX_t = dB_t$. The solution is explicitly given by $Y_t = \cos(X_t + \int_0^t X_s ds)$. This example is inspired by a two person zero sum game from \cite{pham_zhang}.
We choose the following parameters to implement our algorithm: $X_0 = 0$, $\uu{\mu} = 0.2$, $\oo{\mu} = 0.3$, $T = 1$. As illustrated in Figure \ref{fig:nonlinear} and Table \ref{Tab:nonlinear}, with an increase of number of segmentation $\tilde{n}$ and number of time steps $n$ in the Euler scheme will simultaneously improve the accuracy. With only 20 segmentations $(\tilde{n}=20)$ for  $n = 1000$, $Y_0$ reaches 0.9982 with an error of only 0.18\%, where true solution is 1.

\begin{figure}[htbp]
\begin{center}
  \includegraphics[scale=0.45]{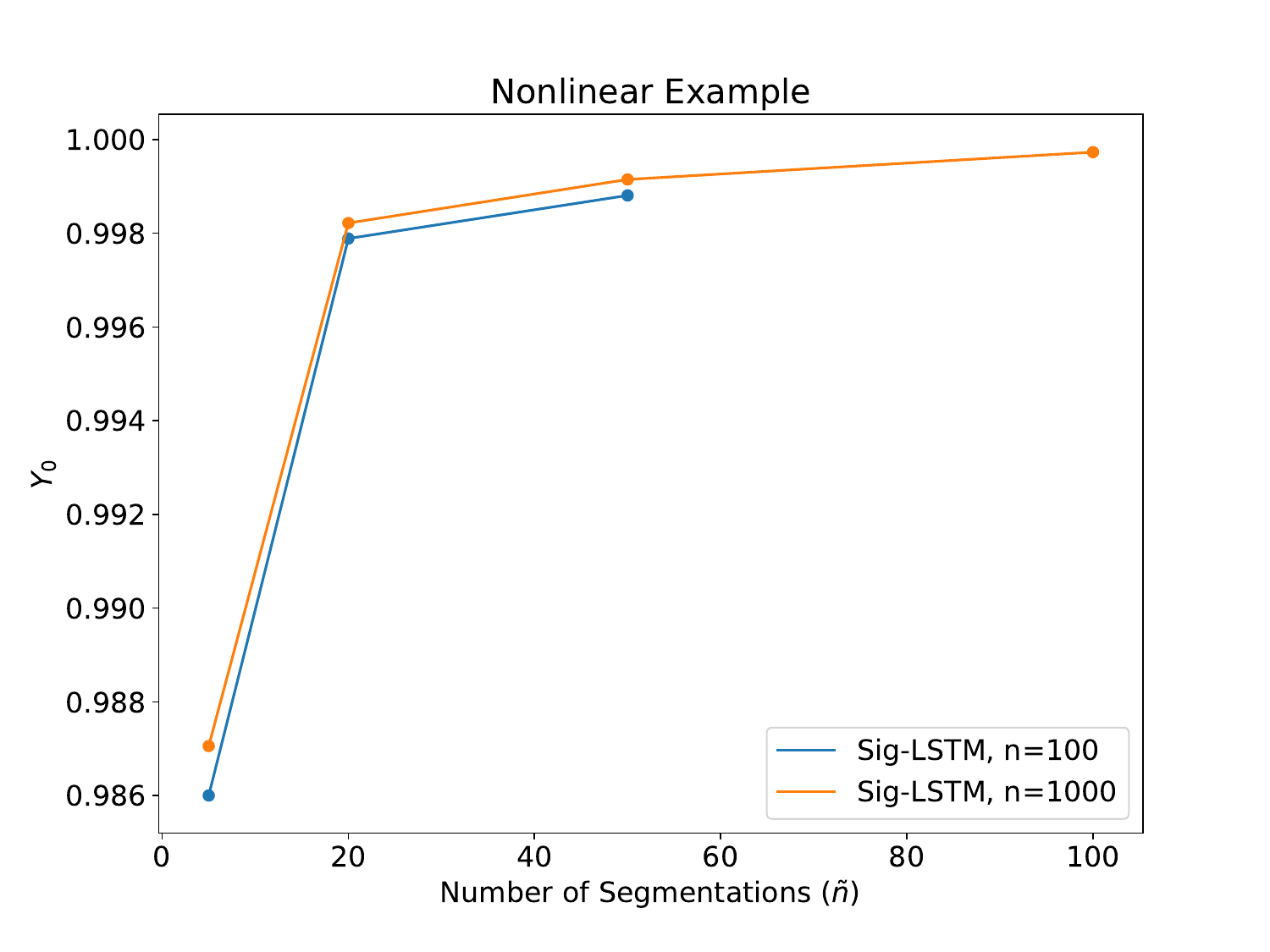}
  \caption{Nonlinear Example.}
  \label{fig:nonlinear}
    \end{center}
\end{figure}

\begin{center}
\captionof{table}{$Y_0$ in the Nonlinear Example\label{Tab:nonlinear}}
\resizebox{0.6\textwidth}{!}{
\begin{tabular}{ |c|c|c|c|c|c| } 
 \hline
  & $\tilde{n}=5$ & $\tilde{n}=20$ & $\tilde{n}=50$ & $\tilde{n}=100$ \\ \hline
$n=100$ &  0.986 & 0.9979 &  0.9988 & --\\ \hline
$n=1000$ &  0.987 & 0.9982 &  0.9991 & 0.9997 \\ 
 \hline
\end{tabular}}
\end{center}

For this non-linear example, the equivalent PPDE follows 
\bea\begin{cases}
	& \partial_t u + \frac{1}{2} \partial_{\omega \omega} u + \min_{\mu \in [\uu{\mu}, \oo{\mu}]} \mu \partial_{\omega} u  + f_0(t, \omega, \bar{\omega}) = 0\\
	& u(T, \omega) = g(\omega_T, \bar{\omega}_T),
\end{cases}\eea 
where $g(\omega_T, \bar{\omega}_T) = \cos(\omega_T + \bar{\omega}_T)$, and $\bar{\omega}_t := \int_0^t \omega_s ds$, and $$f_0(t, \omega, \oo{\omega}) = \max(\sin(\omega_t + \oo{\omega}_t), 0) (\oo{\mu} + \omega_t) + \min(\sin(\omega_t + \oo{\omega}_t), 0) (\uu{\mu} + \omega_t) + \frac{1}{2} \cos(\omega_t + \oo{\omega}_t). $$

\section{Convergence analysis}
In this section, we study the universality approximation property of the Markov FBSDE \eqref{markov FBSDE} and the non-Markovian FBSDE \eqref{non-Markov FBSDEs} by using the deep (log) signature and DNN in the standard Euler schemes. For notation simplicity, we may carry out the proofs only for one-dimensional case, i.e. $d = d_1 = d_2 = 1$.  Before we show the main estimates, we first introduce the following universality property for neural network from \cite{funahashi1993approximation}, see also \cite{liao2019learning}.
\begin{lem}
\label{NN approx}
Let $\hat\sigma(x)$ be a sigmoid function (i.e. a non-constant, increasing, and bounded continuous function on $\hR$). Let K be any compact subset of $\hR^{n}$, and $\hat f: K \rightarrow \hR^{d}$ be a continuous function mapping. Then for an arbitrary $\varepsilon > 0$, there exists an integer $N > 0$, an $d \times N$ matrix A and an N dimensional vector $\theta$ such that
\beaa
\max_{x \in K}|\hat f(x) - A\hat \sigma(Bx + \theta)| \leq \varepsilon,
\eeaa
holds where $\hat\sigma: \hR^{N} \rightarrow \hR^{N}$ is a sigmoid mapping defined by 
\beaa
\hat\sigma('(u_{1}, \cdots u_{N})) = '(\hat \sigma(u_{1}), \cdots, \hat \sigma(u_{N})).
\eeaa
where $'(u_{1}, \cdots u_{N})$ denotes the permutation of the sequence $(u_{1}, \cdots u_{N})$.
\end{lem}

For the time horizon $[0,T]$, we denote $h:= (T-t)/n$ as the step size for the standard Euler scheme, and we denote $t_i := ih, i = 0, \cdots, n$. Similarly, for some $k\in \mathbb R^+$, we denote $u :=k\times(T-t)/n$ as the step size for the deep signature Euler scheme, and we denote $u_i := iu = ikh $, with $\Delta u_i= u_i - u_{i-1} = kh$, for $i = 0, \cdots, n/k:=\widetilde{n}$. Furthermore, we keep the convention that $\Delta W_{i+1}:= W_{i+1} - W_i$ and $\Delta W_{u_{i+1}}:=W_{u_{i+1}}-W_{u_i}$.
\subsection{Markovian case}
\label{convergence Markov}
According to deep signature  Euler scheme \eqref{eqn: Y_sig}, we have
\bea
Y^{\tilde n,\sig}_{u_i} := Y^{\tilde n,\sig}_{u_{i-1}} - f(u_{i-1},X^n_{u_{i-1}}, Y^{\tilde n,\sig}_{u_{i-1}}, Z^{\theta,\sig}_{u_{i-1}}) \Delta u_i + Z^{\theta,\sig}_{u_{i-1}}\Delta W_{u_{i}}. \notag
\eea
where $Z^{\theta,\sig}_{u_i}:=\mathcal R^{\theta}(\pi_m(\sig(X^{n})_0),\cdots,\pi_m(\sig(X^{n})_{i-1}) )$. At last, we denote  $Y^n_{u_i}$ and $Z^n_{u_i}$ as values for the standard Euler scheme approximation of $Y$ and $Z$ (equation \eqref{markov FBSDE}) at time $u_{i}$. The following estimate is a standard result for Markov BSDEs, see \cite{zhang2017backward}[Theorem 5.3.3].
\begin{lem}
\label{Euler approx}
Let Assumptions \ref{FBSDE Assumption} hold and assume h is small enough. Then
\bea
\max_{0 \leq i \leq n}\hE\Big[\sup_{t_{i}\leq t\leq t_{i + 1}}|Y_{t} - Y_{t_{i}}|^{2}\Big]  + \sum_{i = 0}^{n - 1}\hE\Big[\int^{t_{i + 1}}_{t_{i}}|Z_{t} - Z_{t_{i}}|^{2}dt\Big] \leq C[1 + |x|^{2}]h. \notag
\eea
\end{lem}
With the above lemma in hand, we are ready to prove the universal approximation property.

\begin{proof}[Proof of Lemma \ref{Zsig approx}]
We assume that constant $C$ changes generically from line to line. Applying the triangle inequality, for $t\in [u_{i},u_{i+1}]$, we have
$$|Z_t - Z_{u_i}^{\theta, \sig}|^2 \leq 2[|Z_t - Z_{u_i}|^2 + |Z_{u_i} - Z_{u_i}^{\theta, \sig}|^2],$$ 
which implies 
\beaa
\sum_{i = 0}^{\tilde{n}-1}\hE\Big[\int^{u_{i + 1}}_{u_{i}}|Z_{t} - Z_{u_{i}}^{\theta, \sig}|^{2}dt\Big] \leq 2\Big\{\sum_{i = 0}^{\tilde{n}-1}\hE\Big[\int^{u_{i + 1}}_{u_{i}}|Z_{t} - Z_{u_{i}}|^{2}dt+\int^{u_{i + 1}}_{u_{i}}|Z_{u_i} - Z_{u_{i}}^{\theta, \sig}|^{2}dt\Big]\Big\}. \notag
\eeaa
By Lemma \eqref{Euler approx}, we can obtain that
\bea
\sum_{i = 0}^{\tilde{n}-1}\hE\Big[\int^{u_{i + 1}}_{u_{i}}|Z_{t} - Z_{u_{i}}|^{2}dt\Big] \leq C(1 + |x|^2)kh\leq C(1 + |x|^2)\delta.\notag
\eea
 Furthermore, since $\tilde{n} \times kh = T$, we observe that 
\beaa 
\sum_{i = 0}^{\tilde{n}-1}\hE\Big[\int^{u_{i + 1}}_{u_{i}}|Z_{u_i} - Z^{\theta, \sig}_{u_{i}}|^{2}dt\Big]\leq T\max_{0 \leq i \leq \tilde{n}-1}\hE[|Z_{u_i} - Z_{u_i}^{\theta, \sig}|^2]. 
\eeaa

Now it suffices to show that for any $\varepsilon$, there exists network $\theta$ such that  $\max_{0\leq i\leq \tilde{n}-1} \hE[|Z_{u_i} - Z^{\theta, \sig}_{u_{i}}|^{2}] \leq \varepsilon /2T$. From the non-linear Feynman-Kac formula, we know 
\beaa 
Z_t=\pa_xu\sigma(t,X_t):=\mathcal{F}(t,X_t).
\eeaa 
For simplicity, we further assume that the solution $u\in \mathcal C^{\infty,\infty}_b$ and $\sigma\in \mathcal C^{\infty}_b$, which implies that $\mathcal F\in\mathcal C^{\infty,\infty}_b$. We thus have the following Taylor expansion, 
\bea 
dZ_t=d\mathcal{F}(t,X_t)=\pa_t\mathcal{F}(t,X_t)dt+\pa_x\mathcal F(t,X_t)\circ dX_t.
\eea 
Applying the change of variable formula iteratively, we get the following local approximation by using Taylor expansion at step $N$, 
\bea\label{taylor expansion state}
Z_t-Z_s=\mathcal F(t,X_t)-\mathcal F(s,X_s)\approx \sum_{k=1}^N\mathcal F^{\circ k}(\widehat X_s)\int_{I_k}\circ d\widehat X_{t_1}\otimes\cdots d\widehat X_{t_k}
\eea
where we denote $I_k:=\{s< t_1<t_2<\cdots<t_k< t\}$ as the subdivision of the time interval $[s,t]$, and $\{\widehat X_t\}_{t\in[0,T]}:=\{t,X_t\}_{t\in[0,T]}$ as the enhanced path of the time parameter and the path $\{X_t\}_{t\in[0,T]}$. The coefficient term $\mathcal F^{\circ k}$ in the above Taylor expansion is defined recursively, 
\beaa 
\mathcal F^{\circ 1}=\mathcal F=:\pa_xu\sigma, \quad \mathcal F^{\circ k+1}=D(\mathcal F^{\circ k}),
\eeaa 
where $D$ denotes the differential operator. Following the idea in \cite{liao2019learning}[Section 4], we consider the step-$N$ Taylor expansion of $Z$, denoted as $\{\widehat Z_{u_i}\}_{i=0}^{\tilde{n}}$. We have the following approximation of $Z_{u_i}$, 
\beaa 
Z_{u_i}=\mathcal F(u_i,X_{u_i})\approx\widehat Z_{u_i}&=&  \widehat Z_{u_{i-1}}+\sum_{k=1}^N\mathcal F^{\circ k}(\widehat X_{u_{i-1}})\widehat X^k_{u_{i-1},u_i}\\ 
&=& g^{\mathcal F}_N(\sig _{u_{i-1}},\widehat Z_{u_{i-1}}),\\ 
\text{or}&=& \widetilde g^{\mathcal F}_N(\LS_{u_{i-1}},\widehat Z_{u_{i-1}}),
\eeaa 
where $\LS_{u_{i-1}}$ is the log-signature layer of $\widehat X$, and $\sig_{u_{i-1}}$ is the signature layer of $\widehat X$. Plugging in $Z_{u_i}^{\theta, \sig} =\mathcal R^{\theta}((\sig_k)_{k=0}^{u_{i-1}})$ (or $Z_{u_i}^{\theta, \LS} =\mathcal R^{\theta}((\LS_k)_{k=0}^{u_{i-1}})$), for any $u_i$, we have
\beaa
\label{rnn z}
|Z_{u_i} - Z_{u_i}^{\theta, \sig}| &\le& |Z_{u_i} -\widehat Z_{u_i}|+|\widehat Z_{u_i}- Z_{u_i}^{\theta, \sig}| \\
&\le & |Z_{u_i}-\widehat Z_{u_i} |+ |g^{\mathcal F}_N(\sig_{u_{i-1}},\widehat Z_{u_{i-1}}) - Z_{u_i}^{\theta, \LS}|\\
\text{or}~\quad |Z_{u_i} - Z_{u_i}^{\theta, \LS}|&\le & |Z_{u_i}-\widehat Z_{u_i} |+ |\widetilde g^{\mathcal F}_N(\LS_{u_{i-1}},\widehat Z_{u_{i-1}}) - Z_{u_i}^{\theta, \sig}|.
\eeaa
 Applying Lemma \ref{NN approx}, for any $\varepsilon > 0$, as long as $kh=T/\widetilde n$ is small enough and the truncation order of the signature is large enough, we can always find a $\theta$ such that $\max_{i = 0}^{\tilde{n}-1}\hE[|Z_{u_i} - Z^{\theta, \sig}_{u_{i}}|^{2}] \leq \varepsilon/2T$, for any given $T$. In particular, $\varepsilon$ is independent of time discretization. If we replace the signature with log-signature, the proof follows similarly. (A similar proof for forward SDE can be found in \cite{liao2019learning}[Theorem 4.1]).
\qed
\end{proof}


\begin{proof}[Proof of Theorem \ref{main thm}]
Applying triangle inequality, we have 
\[|Y_t - Y_{u_i}^{\tilde n, \sig}|^2 \leq C[|Y_t - Y_{u_i}|^2 + |Y_{u_i} - Y_{u_i}^{\tilde n, \sig}|^2].\] 
According to Lemma \ref{Euler approx}, one can obtain that
\bea
\label{DY0}
\max_{0\leq i\leq \tilde{n}}\hE[\sup_{u_i \leq t \leq u_{i+1}}|Y_t - Y_{u_i}|^2] \leq C(1 + |x|^2)kh\leq C(1 + |x|^2)\delta.
\eea
Next, it suffices to show that $ \max_{0\leq i\leq \tilde{n}}\hE[|Y_{u_i} - Y^{n, \sig}_{u_i}|^2] \leq C(1 + |x|^2)kh$. We denote $\D Y^{\tilde n, \sig}_{u_i} := Y_{u_i} - Y_{u_i}^{\tilde n, \sig}$, and $I_t^{u_i} := f(t, X_t, Y_t, Z_t) - f(u_{i}, X_{u_i}, Y^{\tilde n, \sig}_{u_i}, Z^{\theta, \sig}_{u_i})$. We thus have
\beaa
\label{DY}
\D Y_{u_{i+1}}^{\tilde n, \sig} = \D Y_{u_i}^{\tilde n, \sig} - \int_{u_i}^{u_{i+1}} I^{u_i}_tdt +  \int_{u_i}^{u_{i+1}} Z_t -Z_{u_i}^{\theta, \sig}dW_t.
\eeaa
Taking squares on both sides and taking expectation, we have
\beaa
\label{DY1}
\hE[|\D Y_{u_{i+1}}^{\tilde n, \sig}|^2] \leq C\hE\Big[|\D Y_{u_{i}}^{\tilde n, \sig}|^2 + \int_{u_i}^{u_{i+1}}|I^{u_i}_t|^2dt + \int_{u_i}^{u_{i+1}} |Z_t -Z_{u_i}^{\theta, \sig}|^2 dt\Big].
\eeaa
According to Assumption \ref{FBSDE Assumption}, we further conclude that
\bea
\label{DY2}
|I_t^{u_i}|^2 &\leq& |f(t, X_t, Y_t, Z_t) - f(u_i, X_{u_i}, Y_{u_{i}}, Z_{u_i})|^2 \notag\\
&&+|f(u_i, X_{u_i}, Y_{u_{i}}, Z_{u_i}) -f(u_i, X_{u_i}, Y^{\tilde n, \sig}_{u_i}, Z^{\theta, \sig}_{u_i}) |^2 \notag\\
&\leq& L[(kh)^2 + |Y_t - Y_{u_{i}}|^2 + |X_t - X_{u_i}|^2 + |Z_t - Z_{u_i}|^2 \notag\\
&&+ |\D Y_{u_{i}}^{\tilde n, \sig}|^2 + |Z_{u_i} - Z_{u_i}^{\theta, \sig}|^2].
\eea
Plugging \eqref{DY2} into previous estimates, we obtain
\begin{eqnarray*}
\label{DY3}
\hE\Big[  \int_{u_i}^{u_{i+1}}|I^{u_i}_t|^2dt \Big] &\leq&  \hE\Big[L\int_{u_i}^{u_{i+1}}|Y_t - Y_{u_{i}}|^2 
+ |X_t - X_{u_i}|^2 + |Z_t - Z_{u_i}|^2 dt\notag\\
&&\quad\quad + L(kh)^3 + Lkh|\D Y_{u_{i}}^{\tilde n, \sig}|^2 + Lkh|Z_{u_i} - Z_{u_i}^{\theta, \sig}|^2\Big].  
\end{eqnarray*}
Applying Lemman \ref{Euler approx} and Lemma \ref{Zsig approx}, we further get the following estimates
\beaa
\label{DY4}
\hE[  \int_{u_i}^{u_{i+1}}|I^{u_i}_t|^2dt ]&\leq& C(1 + |x|^2)kh+L(kh)^2 +Lkh\varepsilon+ Lkh\hE[|\D Y_{u_{i}}^{\tilde n, \sig}|^2] \\
&\leq& C(1 + |x|^2+\varepsilon)kh + Lkh\hE[|\D Y_{u_{i}}^{\tilde n, \sig}|^2].
\eeaa
Combining the above estimates and Lemma \ref{Zsig approx}, for some constants $C_1$ and $C_2$, we have
\bea
\label{DY5}
\hE[|\D Y_{u_{i+1}}^{\tilde n, \sig}|^2] \leq  C_1\hE[|\D Y_{u_{i}}^{\tilde n, \sig}|^2] + C_2(1 + |x|^2+\varepsilon)kh.
\eea
Then by \eqref{DY5} and Gr\"onwall's inequality, we can conclude that
\bea
\label{DY6}
\max_{0\leq i\leq \tilde{n}}\hE[|Y_{u_i} - Y^{\tilde n, \sig}_{u_i}|^2] \leq C(1 + |x|^2+\varepsilon)kh.
\eea
At last, combining \eqref{DY0} and \eqref{DY6}, we have
\bea
\max_{0 \leq i \leq \tilde{n}}\hE[\sup_{u_{i} \leq t \leq u_{i+1}}|Y_{t} - Y^{\tilde n,\sig}_{u_{i}}|^{2}] \leq C[1+|x|^2+\varepsilon]\delta. \notag
\eea
\qed
\end{proof}

\subsection{Non-Markovian case}\label{convergence analysis NM FBSDE}
We first introduce the following notations (see e.g. \cite{dupire2019functional}).  Let $T>0$ be fixed. 
Denote $\Omega:= \{\omega \in C([0,T], \mathbb{R}^d)\}$ as the canonical space, and denote $\Lambda:= [0,T] \times \Omega$. For simplification, we still consider $d=1$ below. For each $\omega \in \Omega$, $X: \Lambda \to \mathbb{R}$ is the canonical process, namely $X_t(\omega):= \omega_t$. 
For each $t< t'$ and $(t,\omega)$, $(t', \omega')$ $\in \Lambda$, we denote
\begin{equation}
\left\|\omega\right\|_{t}:=\sup _{s \in[0, t]}|\omega_s|,\quad d_{\infty}\left((t,\omega), (t',\omega')\right):=\sup _{s \in[0, T]}|\omega_{s\wedge t}-\omega'_{s\wedge t'}|+|t-t'| .
\end{equation}
Then  ($\Omega$, $\|\cdot\|$) 
is a Banach space, and $(\Lambda,d_{\infty})$ is a complete pseudometric space. For a function $u:\Lambda \rightarrow \hR $, the path derivatives of $u$ are defined as, if they exist, 
\begin{equation}\label{path derivative}
\begin{aligned}
D_tu(t,\omega_{\cdot\wedge t})&=\lim_{h\rightarrow 0, h>0}\frac{1}{h}[u(t+h,\omega_{\cdot\wedge t})-u(t,\omega_{\cdot\wedge t}) ],\\
D_{\omega}u(t,\omega_{\cdot\wedge t})&=\lim_{h\rightarrow 0, h>0}\frac{1}{h}[u(t,\omega_{\cdot\wedge t}+h\mathbf{1}_{[t,T]})-u(t,\omega_{\cdot\wedge t})]. 
\end{aligned}
\end{equation}
Similarly, we define $D_{\omega\omega}u(t,\omega_{\cdot\wedge t})=D_{\omega}(D_{\omega} u(t,\omega_{\cdot\wedge t}) )$. We first introduce the following functional It\^o's formula from \cite{dupire2019functional}.
\begin{thm}
Let $\left(\widetilde \Omega, \mathcal{F},\left(\mathcal{F}_{t}\right)_{t \in[0, T]}, \hP\right)$ be a probability space, if $X$ is a continuous semi-martingale and $u$ is in $\mathbb{C}^{1,2}(\Lambda)$, then for any $t \in[0, T]$ :
\begin{equation}\label{functional ito}
    u\left(X_{t}\right)-u\left(X_{0}\right)=\int_{0}^{t} D_{s} u\left(X_{s}\right) d s+\int_{0}^{t} D_{x} u\left(X_{s}\right) d X_s+\frac{1}{2} \int_{0}^{t} D_{xx} u\left(X_{s}\right) d\langle X\rangle_s, \quad a . s . .
\end{equation}
\end{thm}
For the purpose of our analysis, we record the Stratonovich form of the above It\^o's formula \eqref{functional ito} as below, 
\begin{equation}\label{stratonovich path}
u\left(X_{t}\right)-u\left(X_{0}\right)=\int_{0}^{t} D_{s} u\left(X_{s}\right) d s+\int_{0}^{t} D_{x} u\left(X_{s}\right) \circ d X_s.
\end{equation}
Now we are ready to prove the convergence of the non-Markovian FBSDE algorithm. We first introduce the following assumption. 
\begin{assumption}
\label{FBSDE Assumption path}
Let the following assumptions be in force.
\begin{itemize}
\item $b, \sigma, f, g$ are deterministic taking values in $\hR^{d_1},~ \hR^{d_1 \times d},~ \hR^{d_2},~ \hR^{d_2}$, respectively; and $b(\cdot, 0), \sigma(\cdot, 0), f(\cdot, 0, 0, 0)$ and $g(0)$ are bounded.
\item  $b, \sigma, f, g$ are $C^k$-smooth enough  with respect to all variables $(t,x_{\cdot\wedge t}, y, z)$ for any desired $k\in\mathbb N_+$  and all derivatives are bounded by constant $L$.
\end{itemize}
\end{assumption}
We denote  $\bar Y^n_{u_i}$ and $\bar Z^n_{u_i}$ as values for the standard Euler scheme approximation of $Y$ and $Z$ at time $u_{i}$ in equation \eqref{non-Markov FBSDEs}. According to Theorem \ref{functional ito}, the estimate from Lemma \ref{Euler approx} holds true for the standard Euler scheme of \eqref{non-Markov FBSDEs} under Assumption \ref{FBSDE Assumption path}.
\begin{lem}
\label{Euler approx path}Let Assumptions \ref{FBSDE Assumption path} hold and assume h is small enough. Then
\bea
\max_{0 \leq i \leq n}\hE\Big[\sup_{t_{i}\leq t\leq t_{i + 1}}|Y_{t} - \bar Y_{t_{i}}|^{2}\Big]  + \sum_{i = 0}^{n - 1}\hE\Big[\int^{t_{i + 1}}_{t_{i}}|Z_{t} - \bar Z_{t_{i}}|^{2}dt\Big] \leq C[1 + |x|^{2}]h. \notag
\eea
\end{lem}

Next, we prove the path-dependent version of Lemma \ref{Zsig approx} and Theorem \ref{main thm}. \remove{for} \eqref{eqn: Y_sig}

\begin{lem}
\label{Zsig approx path}
Let \textbf{Assumption \ref{FBSDE Assumption path}} hold and assume $kh<\delta$, for any $\delta>0$, for any given $T>0$, for some constant $C>0$ depending on $T$ and $L$, and for any $\varepsilon>0$, there exists recurrent neural network $\mathcal R^{\theta}$, such that
\bea
\sum_{i = 0}^{\tilde{n}-1}\hE\Big[\int^{u_{i + 1}}_{u_{i}}|Z_{t} -\bar Z_{u_{i}}^{\theta, \sig} |^{2}dt\Big] \leq C[1 + |x|^{2}]\delta+\varepsilon. \notag
\eea
\end{lem}
\begin{proof}
According to \eqref{eqn: Y_sig}, we have 
\begin{equation}
\begin{aligned}
      \bar  Y^{\tilde{n},\sig}_{u_i} :=  & \bar Y^{\tilde{n},\sig}_{u_{i-1}} - f(u_{i-1},X^n_{[0,u_{i-1}]}, \bar Y^{\tilde{n},\sig}_{u_{i-1}}, Z^{\theta,\sig}_{u_{i-1}}) \Delta u_i + \bar Z^{\theta,\sig}_{u_{i-1}}\Delta W_{u_{i}}.
\end{aligned}
\end{equation}
Similar to the proof of Lemma \ref{Zsig approx}, we have 
\beaa
\sum_{i = 0}^{\tilde{n}-1}\hE\Big[\int^{u_{i + 1}}_{u_{i}}|Z_{t} -\bar Z_{u_{i}}^{\theta, \sig}|^{2}dt\Big] \leq 2\Big\{\sum_{i = 0}^{\tilde{n}-1}\hE\Big[\int^{u_{i + 1}}_{u_{i}}|Z_{t} - \bar Z_{u_{i}}|^{2}dt+\int^{u_{i + 1}}_{u_{i}}|\bar Z_{u_i} -\bar  Z_{u_{i}}^{\theta, \sig}|^{2}dt\Big]\Big\}. \notag
\eeaa
Applying Lemma \eqref{Euler approx path}, we have
\bea
\sum_{i = 0}^{\tilde{n}-1}\hE\Big[\int^{u_{i + 1}}_{u_{i}}|Z_{t} - \bar Z_{u_{i}}|^{2}dt\Big] \leq C(1 + |x|^2)kh \leq C(1 + |x|^2)\delta.\notag
\eea
Similar to proof of Lemma \ref{Zsig approx}, we get 
\beaa
\sum_{i = 0}^{\tilde{n}-1}\hE\Big[\int^{u_{i + 1}}_{u_{i}}|\bar Z_{u_i} -\bar Z^{\theta, \sig}_{u_{i}}|^{2}dt\Big]\leq T\max_{0 \leq i \leq \tilde{n}-1}\hE[|\bar Z_{u_i} -\bar Z_{u_i}^{\theta, \sig}|^2].
\eeaa
Next, we show that for any $\varepsilon$, there exists network $\theta$ such that $\max_{0\leq i\leq \tilde{n}-1} \hE[|\bar Z_{u_i} -\bar Z^{\theta, \sig}_{u_{i}}|^{2}] \leq \varepsilon/2T$. Applying the nonlinear Feynman-Kac formula \cite{peng2016bsde} for the non-Markovian BSDE \eqref{non-Markov FBSDEs} and using the definition \eqref{path derivative}, we get (e.g. \cite{peng2016bsde}[Proposition 3.8]),
\beaa 
Z_t=D_xu(t,X_{\cdot\wedge t})\sigma(t,X_t):=\widehat{\mathcal{F}}(t,X_{\cdot\wedge t}), \quad \text{for}\quad X_{\cdot\wedge t}\in\Lambda. 
\eeaa 
According to Assumption \ref{FBSDE Assumption path}, we assume that the functional $\widehat{\mathcal F}\in\mathcal C^{\infty,\infty}_b$ is smooth. Applying \eqref{stratonovich path}, we thus have the following Taylor expansion, 
\bea 
dZ_t=d\widehat{\mathcal{F}}(t,X_{\cdot\wedge t})=D_t\widehat{\mathcal{F}}(t,X_{\cdot\wedge t})dt+D_x\widehat{\mathcal F}(t,X_{\cdot\wedge t})\circ dX_t.
\eea 
Applying the change of variable formula iteratively, we get the following local approximation by using Taylor expansion at step $N$, 
\beaa 
Z_t-Z_s=\widehat{\mathcal F}(t,X_{\cdot\wedge t})-\widehat{\mathcal F}(s,X_{\cdot\wedge s})\approx \sum_{k=1}^N\widehat{\mathcal F}^{\circ k}(\widehat X_{\cdot\wedge s})\int_{I_k}\circ d\widehat X_{t_1}\otimes\cdots d\widehat X_{t_k},
\eeaa 
which is similar to \eqref{taylor expansion state}. The key difference is that  the coefficient term $\widehat{\mathcal F}^{\circ k}$ in the above Taylor expansion is defined recursively as below, 
\beaa 
\widehat{\mathcal F}^{\circ 1}=\widehat{\mathcal F}=:D_xu\sigma, \quad \widehat{\mathcal F}^{\circ k+1}=D(\widehat{\mathcal F}^{\circ k}),
\eeaa  
where the derivative is defined in \eqref{path derivative} following the functional It\^o's formula. Since the signature term 
\[
\int_{I_k}\circ d\widehat X_{t_1}\otimes\cdots d\widehat X_{t_k},\quad \text{for}\quad k\in\hN_+
\]
is identical to the ones in Lemma \ref{Zsig approx}, the rest of the proof follows directly from Lemma \ref{Zsig approx}. The proof is thus completed.
\end{proof}
Applying the above Lemma \ref{Zsig approx path} and following the similar proof of Theorem \ref{main thm}, we have the following estimate.
\begin{thm}\label{main thm path}
Let \textbf{Assumption \ref{FBSDE Assumption path}} be in force. Assume that $kh<\delta$ for any small $\delta>0$, for any given $T>0$, for some constant $C>0$ depnding on $T$ and $L$ in \textbf{Assumption \ref{FBSDE Assumption path}}, and for any $\varepsilon>0$, there exists recurrent neural network $\mathcal R^{\theta}$, such that
\bea
\max_{0 \leq i \leq \tilde{n}}\hE[\sup_{u_{i} \leq t \leq u_{i+1}}|Y_{t} - Y^{\tilde  n,\sig}_{u_{i}}|^{2}] \leq C[1+|x|^2+\varepsilon]\delta. \notag
\eea
\end{thm}

\section{Conclusion}
This paper aims to develop efficient algorithms to solve non-Markovian FBSDEs or equivalent PPDEs. We combine the signature/log-signature transformation together with RNN model to solve the FBSDE numerically. Our algorithms show advantages in solving path-dependent problems, high-frequency data problems, and long time duration problems, which apply to a wide range of applications in financial markets. 

\bigbreak

\noindent\textbf{Acknowledgments.} 
We would like to thank Professor Jin Ma and Professor Jianfeng Zhang  for all the insightful comments.

%

  \end{document}